\documentclass[reqno]{amsart}						
\usepackage{microtype}

\usepackage{amsmath}						
\usepackage{amssymb}						
\usepackage{amsfonts}						
\usepackage{amsthm}	
\usepackage[foot]{amsaddr}						
\usepackage{array}

\usepackage{mathtools}						
\mathtoolsset{%
}

\usepackage[utf8]{inputenc}							
\usepackage[T1]{fontenc}						

\usepackage[proportional,lining]{libertine}
\usepackage[libertine,libaltvw,cmintegrals,varbb]{newtxmath}

\usepackage{dsfont}							

\usepackage[
cal=cm,
]
{mathalfa}

\usepackage[labelfont={bf,small},labelsep=colon,font=small]{caption}	
\usepackage[dvipsnames,svgnames]{xcolor}						
\colorlet{MyBlue}{DodgerBlue!75!Black}
\colorlet{MyGreen}{DarkGreen!85!Black}

\renewcommand{\paragraph}{\subsection}
\newcommand{\afterhead}{.}							

\usepackage{subcaption}
\usepackage{tikz}							
\usetikzlibrary{calc,patterns}

\usepackage{acronym}						
\usepackage{booktabs}						
\usepackage{latexsym}						
\usepackage{paralist}						
\usepackage{xspace}							

\usepackage[sort&compress]{natbib}					

\usepackage{hyperref}
\hypersetup{
final,
colorlinks=true,
linktocpage=true,
pdfstartview=FitH,
breaklinks=true,
pdfpagemode=UseNone,
pageanchor=true,
pdfpagemode=UseOutlines,
plainpages=false,
bookmarksnumbered,
bookmarksopen=false,
bookmarksopenlevel=1,
hypertexnames=true,
pdfhighlight=/O,
urlcolor=Maroon,linkcolor=MyBlue!60!black,citecolor=DarkGreen!70!black,	
pdftitle={},
pdfauthor={},
pdfsubject={},
pdfkeywords={},
pdfcreator={pdfLaTeX},
pdfproducer={LaTeX with hyperref}
}

\def\EMAIL#1{\email{\href{mailto:#1}{\texttt{\upshape #1}}}}

\numberwithin{equation}{section}						
\usepackage[sort&compress,capitalize,nameinlink]{cleveref}			
\crefrangeformat{equation}{\upshape(#3#1#4)\textendash(#5#2#6)}

\DeclarePairedDelimiter{\bracks}{[}{]}					
\DeclarePairedDelimiter{\parens}{(}{)}					

\DeclarePairedDelimiter{\norm}{\lVert}{\rVert}					
\DeclarePairedDelimiterXPP{\dnorm}[1]{}{\lVert}{\rVert}{_{\ast}}{#1}	

\DeclarePairedDelimiterX{\braket}[2]{\langle}{\rangle}{#1,#2}		
\DeclarePairedDelimiterX{\product}[2]{\langle}{\rangle}{#1,#2}		

\DeclarePairedDelimiterX{\setdef}[2]{\{}{\}}{#1:#2}					
\DeclarePairedDelimiterXPP{\exclude}[1]{\mathopen{}\setminus}{\{}{\}}{}{#1}

\newcommand{\cf}{cf.\xspace}							
\newcommand{\eg}{e.g.,\xspace}						
\newcommand{\ie}{i.e.,\xspace}						
\newcommand{\textpar}[1]{\textup(#1\textup)}					

\newcommand{\txs}{\textstyle}						

\newcommand{\acc}[1][x]{\hat#1}						
\newcommand{\alt}[1]{#1'}							
\newcommand{\dual}[1]{#1^{\ast}}						
\newcommand{\est}[1]{\hat #1}						
\newcommand{\sol}[1][\state]{#1^{\ast}}					

\DeclareMathOperator{\bigoh}{\mathcal O}						

\newcommand{\R}{\mathbb{R}}						

\newcommand{\vdim}{\debug d}							
\newcommand{\vecspace}{\mathcal{\debug V}}					
\newcommand{\dspace}{\mathcal{\debug Y}}					

\newcommand{\state}{\debug x}							
\newcommand{\dstate}{\debug y}							

\DeclareMathOperator{\diam}{diam}					

\newcommand{\base}{\debug p}							

\newcommand{\cpt}{\mathcal{\debug C}}						

\newcommand{\mat}{\debug M}							

\newcommand{\subd}{\partial}						

\newcommand{\feas}{\mathcal{\debug X}}						
\newcommand{\intfeas}{\feas^{\circ}}					
\newcommand{\sols}{\sol[\feas]}
\newcommand{\Lip}{\debug L}							
\newcommand{\obj}{\debug f}							
\newcommand{\gvec}{\debug g}							

\newcommand{\play}{\debug i}							

\newcommand{\nPures}{\debug A}							
\newcommand{\pures}{\mathcal{\nPures}}					

\newcommand{\payv}{\debug v}							

\newcommand{\eq}{\sol}							
\newcommand{\fingame}{\debug \Gamma}						

\DeclareMathOperator{\Eucl}{\debug \Pi}						

\newcommand{\breg}{\debug D}							
\newcommand{\hreg}{\debug h}							
\newcommand{\mirror}{\debug Q}							
\newcommand{\prox}{\debug P}							
\newcommand{\strong}{\debug K}							

\newcommand{\new}[1]{#1^{+}}						
\newcommand{\act}{\debug X}							
\newcommand{\step}{\debug \gamma}						

\newcommand{\dynfield}{\debug V}							

\DeclareMathOperator{\ex}{\mathbb{E}}					
\DeclareMathOperator{\prob}{\mathbb{P}}				
\DeclareMathOperator{\simplex}{\Delta}					

\newcommand{\as}{\textpar{a.s.}\xspace}				
\newcommand{\dkl}{D_{\mathrm{KL}}}					
\newcommand{\filter}{\mathcal{\debug F}}						
\newcommand{\noise}{\debug U}							
\newcommand{\snoise}{\debug \xi}							
\newcommand{\noisedev}{\debug \sigma}						
\newcommand{\noisevar}{\noisedev^{2}}					

\providecommand\given{}							

\DeclarePairedDelimiterXPP{\exof}[1]{\ex}{[}{]}{}{
\renewcommand\given{\nonscript\,\delimsize\vert\nonscript\,\mathopen{}} #1}

\DeclarePairedDelimiterXPP{\probof}[1]{\prob}{(}{)}{}{
\renewcommand\given{\nonscript\,\delimsize\vert\nonscript\,\mathopen{}} #1}

\DeclareMathOperator*{\argmax}{arg\,max}						
\DeclareMathOperator*{\argmin}{arg\,min}						

\DeclareMathOperator{\dom}{dom}						

\newcommand{\from}{\colon}							

\newcommand{\start}{\debug 1}							
\newcommand{\running}{\debug 1,2,\dotsc}						
\newcommand{\run}{\debug n}							
\newcommand{\runalt}{\debug k}							

\usepackage[textwidth=30mm]{todonotes}					

\newcommand{\debug}[1]{#1}						
\newcommand{\revise}[1]{#1}						

\usepackage{algorithm}							
\usepackage{algpseudocode}							

\theoremstyle{plain}
\newtheorem{theorem}{Theorem}						
\newtheorem*{theorem*}{Theorem}						
\newtheorem{corollary}[theorem]{Corollary}					
\newtheorem*{corollary*}{Corollary}						
\newtheorem{lemma}[theorem]{Lemma}					
\newtheorem{proposition}[theorem]{Proposition}				

\theoremstyle{definition}
\newtheorem{definition}[theorem]{Definition}				
\newtheorem*{definition*}{Definition}					
\newtheorem*{assumption*}{Assumptions}					

\theoremstyle{remark}
\newtheorem*{remark*}{Remark}						
\newtheorem{example}{Example}						
\newtheorem*{example*}{Example}						

\numberwithin{equation}{section}						
\numberwithin{theorem}{section}						
\numberwithin{remark}{section}						
\numberwithin{example}{section}						

\newcommand{\gbound}{\debug G}							
\newcommand{\farbound}{\debug a}							

\begin{document}


\title[Optimistic Mirror Descent in Saddle-Point Problems]
{Optimistic Mirror Descent in Saddle-Point Problems:\\
Going the Extra (Gradient) Mile}

\author
[P.~Mertikopoulos]
{Panayotis Mertikopoulos$^{\ast}$}
\address{$^{\ast}$
Univ. Grenoble Alpes, CNRS, Inria, LIG, 38000 Grenoble, France.}
\EMAIL{panayotis.mertikopoulos@imag.fr}

\author
[B.~Lecouat]
{Bruno Lecouat$^{\ddag}$}
\EMAIL{bruno\textunderscore lecouat@i2r.a-star.edu.sg} 

\author
[H.~Zenati]
{Houssam Zenati$^{\ddag}$}
\address{$^{\ddag}$ Institute for Infocomm Research, A*STAR, 1 Fusionopolis Way, \#21-01 Connexis (South Tower), Singapore.}
\EMAIL{houssam\textunderscore zenati@i2r.a-star.edu.sg} 

\author
[C.-S.~Foo]
{\\Chuan-Sheng Foo$^{\ddag}$}
\EMAIL{foocs@i2r.a-star.edu.sg}

\author
[V.~Chandrasekhar]
{Vijay Chandrasekhar$^{\ddag \P}$}
\address{$^{\P}$ Nanyang Technological University, 50 Nanyang Ave, Singapore.}
\EMAIL{vijay@i2r.a-star.edu.sg} 

\author
[G.~Piliouras]
{Georgios Piliouras$^{\S}$}
\address{$^{\S}$ Singapore University of Technology and Design, 8 Somapah Road, Singapore.}
\EMAIL{georgios@sutd.edu.sg}

\thanks{
%
%
P.~Mertikopoulos was partially supported by
the French National Research Agency (ANR) grant
ORACLESS (ANR\textendash 16\textendash CE33\textendash 0004\textendash 01).
G.~Piliouras would like to acknowledge SUTD grant SRG ESD 2015 097, MOE AcRF Tier 2 Grant  2016-T2-1-170
and NRF fellowship NRF-NRFF2018-07. This work was partly funded by the deep learning 2.0 program at A*STAR.}


\newcommand{\acdef}[1]{\textit{\acl{#1}} \textup{(\acs{#1})}\acused{#1}}		
\newcommand{\acdefp}[1]{\emph{\aclp{#1}} \textup(\acsp{#1}\textup)\acused{#1}}	

\newacro{MWU}{multiplicative weights update}
\newacro{AI}{artificial intelligence}
\newacro{GMM}{Gaussian mixture model}
\newacro{PCA}{principal component analysis}
\newacro{SP}{saddle-point}
\newacro{FTRL}{follow-the-regularized-leader}
\newacro{OMD}{optimistic mirror descent}
\newacro{OGD}{optimistic gradient descent}
\newacro{GD}{gradient descent}
\newacro{MD}{mirror descent}
\newacro{SPMD}{saddle-point mirror descent}
\newacro{DA}{dual averaging}
\newacro{DGF}{distance-generating function}
\newacro{MW}{multiplicative weights}
\newacro{KL}{Kullback\textendash Leibler}
\newacro{MP}{mirror-prox}
\newacro{SMP}{stochastic mirror-prox}
\newacro{MDS}{martingale difference sequence}
\newacro{GAN}{generative adversarial network}
\newacro{LHS}{left-hand side}
\newacro{RHS}{right-hand side}
\newacro{iid}[i.i.d.]{independent and identically distributed}
\newacroplural{iid}[i.i.d.]{independent and identically distributed}
\newacro{NE}{Nash equilibrium}
\newacroplural{NE}[NE]{Nash equilibria}
\newacro{VI}{variational inequality}
\newacroplural{VI}{variational inequalities}

\begin{abstract}
%
%
Owing to their connection with \acp{GAN}, \acl{SP} problems have recently attracted considerable interest in machine learning and beyond.
By necessity, most theoretical guarantees revolve around convex-concave \revise{(or even linear)} problems;
however, making theoretical inroads towards efficient \ac{GAN} training depends crucially on moving beyond this classic framework.
To make piecemeal progress along these lines, we analyze the behavior of \ac{MD} in a class of non-monotone problems whose solutions coincide with those of a naturally associated \acl{VI} \textendash\ a property which we call \emph{coherence}.
\revise{We first show that ordinary, ``vanilla'' \ac{MD} converges under a strict version of this condition, but not otherwise;
in particular, it may fail to converge even in bilinear models with a unique solution.
We then show that this deficiency is mitigated by optimism:
by taking an ``extra-gradient'' step, \ac{OMD} converges in \emph{all} coherent problems.
Our analysis generalizes and extends the results of \cite{DISZ18} for \ac{OGD} in bilinear problems,
and makes concrete headway for provable convergence beyond convex-concave games.
We also provide stochastic analogues of these results, and we validate our analysis by numerical experiments in a wide array of \ac{GAN} models (including \aclp{GMM}, and the CelebA and CIFAR-10 datasets).}
\end{abstract}

\maketitle
\allowdisplaybreaks							
\acresetall								


\vspace{-2em}
\begin{figure}[h]
\centering
\includegraphics[height=.3\textwidth]{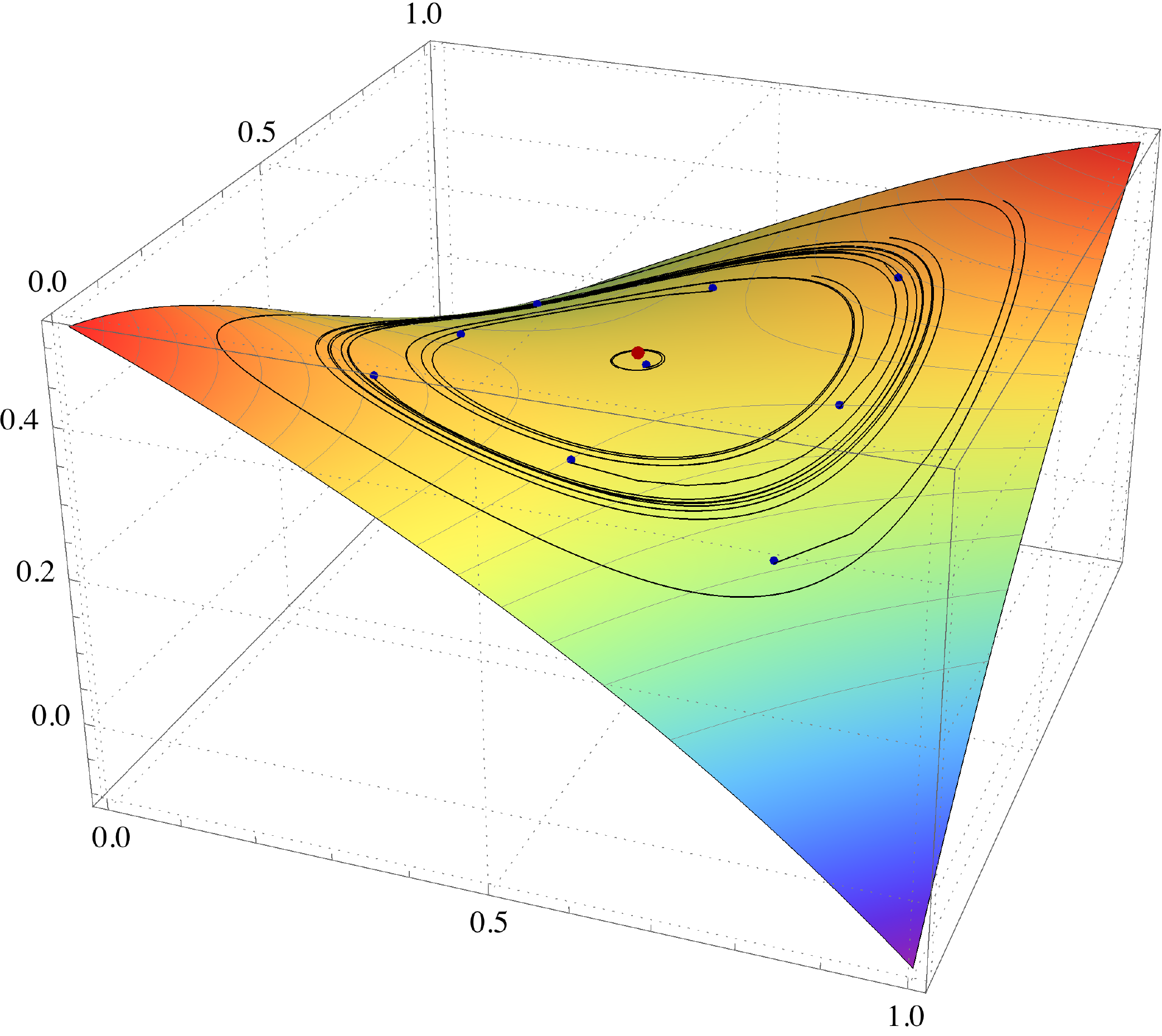}
\qquad
\includegraphics[height=.3\textwidth]{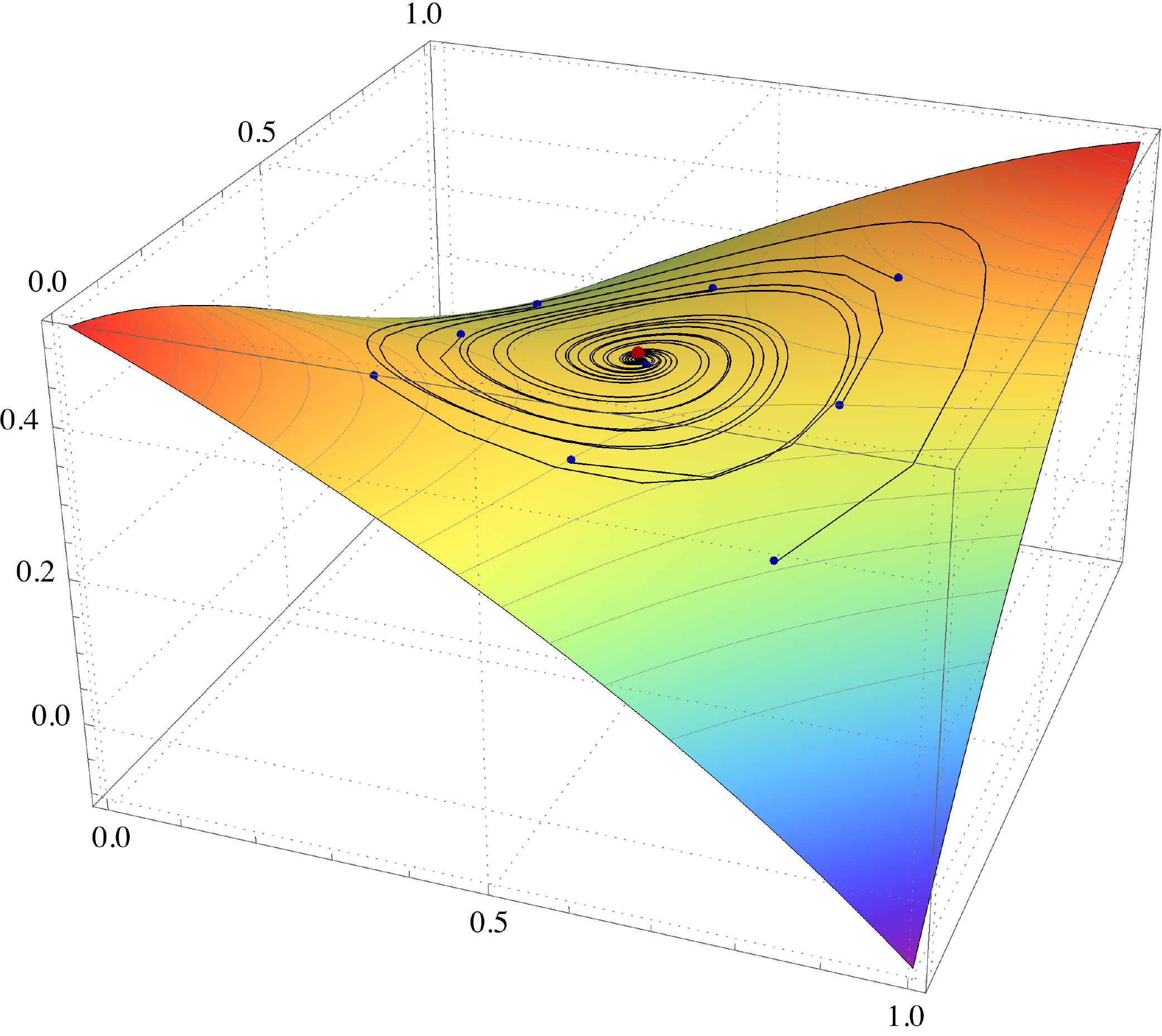}
\caption{%
Mirror descent (\acs{MD}\acused{MD}) in the non-monotone saddle-point problem $\obj(\state_{1},\state_{2}) = (\state_{1} - 1/2) (\state_{2} - 1/2) + \frac{1}{3} \exp(-(\state_{1}-1/4)^{2} - (\state_{2}-3/4)^{2})$.
Left: vanilla \ac{MD} spirals outwards;
right: optimistic \ac{MD} converges.
}%
\label{fig:portraits}
\end{figure}


\section{Introduction}
\label{sec:introduction}

The surge of recent breakthroughs in \ac{AI} has sparked significant interest in solving optimization problems that are universally considered hard.
Accordingly, the need for an effective theory has two different sides:
first, a deeper theoretical understanding would help demystify the reasons behind the success and/or failures of different training algorithms; 
second, theoretical advances can inspire effective algorithmic tweaks leading to concrete performance gains.

Deep learning has been an area of AI where theory has provided a significant boost.
As a functional class, deep learning involves non-convex loss functions for which finding even local optima is NP-hard;
nevertheless, elementary techniques such as gradient descent (and other first-order methods) seem to work fairly well in practice. 
For this class of problems, recent theoretical results have indeed provided useful insights:
using tools from the theory of dynamical systems, \cite{pmlr-v49-lee16,lee2017first} and \cite{panageas2017gradient} showed that a wide variety of first-order methods (including \acl{GD} and \acl{MD}) almost always avoid saddle points.
More generally, the optimization and machine learning communities alike have dedicated significant effort in understanding the geometry of non-convex landscapes by searching for properties which could be leveraged for efficient training.
For example, the well-known ``strict saddle'' property was shown to hold in a wide range of salient objective functions
ranging from
low-rank matrix factorization \citep{ge2016matrix,ge2017no,bhojanapalli2016global}
and
dictionary learning  \citep{sun2015complete2,sun2015complete1},
to
\acl{PCA} \citep{ge2015escaping},
phase retrieval \citep{sun2016phase},
and many other models.

On the other hand, \emph{adversarial} deep learning is nowhere near as well understood,
especially in the case of \acp{GAN} \citep{goodfellow2014generative}.
Despite an immense amount of recent scrutiny, our theoretical understanding cannot boast similar breakthroughs as in the case of ``single-agent'' deep learning.
To make matters worse, \acp{GAN} are notoriously hard to train and standard optimization methods often fail to converge to a reasonable solution.
Because of this, a considerable corpus of work has been devoted to exploring and enhancing the stability of \acp{GAN},
including techniques as diverse as
the use of Wasserstein metrics \citep{wgan}, 
critic gradient penalties \citep{wgan_gp},
different activation functions in different layers,
feature matching,
minibatch discrimination,
etc. \citep{dcgan,improved-gans}.

A key observation in this context is that first-order methods may fail to converge even in toy, bilinear zero-sum games like Rock-Paper-Scissors and Matching Pennies \citep{piliouras2014optimization,CRS16,DISZ18,MGN18,MPP18,bailey2018multiplicative}.
\revise{This is a critical failure of descent methods, but one which \cite{DISZ18} showed can be overcome through ``\emph{optimism}'', interpreted in this context as a momentum adjustment that pushes the training process one step further along the incumbent gradient.
In particular, \cite{DISZ18} showed that \ac{OGD} succeeds in cases where vanilla \ac{GD} fails (specifically, unconstrained bilinear \acl{SP} problems), and leveraged this theoretical result to improve the training of \acp{GAN}.}

A common theme in the above is that, to obtain a principled methodology for training \acp{GAN}, it is beneficial to first establish improvements in a more restricted setting, and then test whether these gains carry over to more demanding learning environments.
Following these theoretical breadcrumbs,
\revise{we focus on a class of non-monotone problems whose solutions coincide with those of a naturally associated \acl{VI}, a property which we call \emph{coherence}.
Then, motivated by the success of \acf{MD} methods in online/stochastic convex programming,
and hoping to overcome the shortcomings of ordinary \acl{GD} by exploiting the problem's geometry,
we examine the convergence of \ac{MD} in coherent problems.
On the positive side, we show that if a problem is \emph{strictly} coherent (a condition that is satisfied by all strictly monotone
problems), \ac{MD} converges almost surely, even in stochastic problems (\cref{thm:MD}).
However, under \emph{null coherence} (the ``saturated'' opposite to strict coherence), \ac{MD} spirals outwards from the problem's solutions and may cycle in perpetuity, even with perfect gradient feedback.
The null coherence property covers all bilinear models, so this result generalizes and extends the recent analysis of \cite{DISZ18} and \cite{bailey2018multiplicative} for \acl{GD} and \ac{FTRL} respectively (for a schematic illustration, see \cref{fig:portraits,fig:portraits-bilinear}).
Thus, in and by themselves, gradient/\acl{MD} methods \emph{do not} suffice for training convoluted, adversarial deep learning models.}

\revise{To mitigate this deficiency, we introduce an extra-gradient step which allows the algorithm to look ahead and take an ``optimistic'' mirror step along a ``future'' gradient.
Following \cite{RS13}, this method is known as \acdef{OMD}, and was first studied under the name ``\acl{MP}'' by \cite{Nem04}.
In convex-concave problems, \cite{Nem04} showed that the so-called ``ergodic average'' of the algorithm's iterates enjoys an $\bigoh(1/\run)$ convergence rate.
In the context of \ac{GAN} training,
\cite{GBVLJ18} further introduced a ``gradient reuse'' mechanism to minimize the computational overhead of back-propagation and proved convergence in stochastic convex-concave problems.
However, beyond the monotone regime, averaging offers no tangible benefits because Jensen's inequality no longer applies;
as a result, moving closer to \acp{GAN} requires changing both the algorithm's output structure as well as the accompanying analysis.}

\revise{Our first result in this direction is that the \emph{last iterate} of \ac{OMD} converges in all coherent problems, \emph{including null-coherent ones}.
As a special case, this generalizes and extends the results of \cite{DISZ18} for \ac{OGD} in bilinear problems, and also settles in the affirmative an issue left open by the authors concerning the convergence of the algorithm in nonlinear problems.
In addition, under the \ac{OMD} algorithm, the (Bregman) distance to a solution decreases \emph{monotonically}, so each iterate is better than the previous one (\cref{thm:OMD}).
Finally, under strict coherence, we also show that \ac{OMD} converges with probability $1$ in stochastic \acl{SP} problems (\cref{thm:OMD-strict}).
These results suggest that a straightforward, extra-gradient add-on can lead to significant performance gains when applied to existing state-of-the-art first-order methods (such as Adam).
This theoretical prediction is validated experimentally in a wide array of \ac{GAN} models (including \aclp{GMM}, and the CelebA and CIFAR-10 datasets) in \cref{sec:experiments}.}

\section{Problem setup and preliminaries}
\label{sec:setup}

\paragraph{Saddle-point problems\afterhead}

Consider a \acl{SP} problem of the general form
\begin{equation}
\label{eq:SP}
\tag{SP}
\min_{\state_{1}\in\feas_{1}} \max_{\state_{2}\in\feas_{2}} \obj(\state_{1},\state_{2}),
\end{equation}
where
each feasible region $\feas_{\play}$, $\play=1,2$, is a compact convex subset of a finite-dimensional normed space $\vecspace_{\play} \equiv \R^{\vdim_{\play}}$,
and
$\obj\from\feas\equiv\feas_{1}\times\feas_{2}\to\R$ denotes the problem's value function.%
\footnote{Compactness is assumed chiefly to streamline our presentation.
The convex-\emph{closed} framework can be dealt with via a coercivity assumption;
however, this would take us too far afield, so we do not pursue this direction.}
From a game-theoretic standpoint, \eqref{eq:SP} can be seen as a \emph{zero-sum game} between two optimizing agents (or \emph{players}):
Player $1$ (the \emph{minimizer}) seeks to incur the least possible loss,
while
Player $2$ (the \emph{maximizer}) seeks to obtain the highest possible reward \textendash\ both given by $\obj(\state_{1},\state_{2})$.

To obtain a solution of \eqref{eq:SP}, we will focus on incremental processes that exploit the individual loss/reward gradients of $\obj$ (assumed throughout to be at least $C^{1}$-smooth).
Since the individual gradients of $\obj$ will play a key role in our analysis, we will encode them in a single vector as
\begin{equation}
\label{eq:paygrad}
\gvec(\state)
	= (\gvec_{1}(\state),\gvec_{2}(\state))
	= (\nabla_{\state_{1}}\obj(\state_{1},\state_{2}),-\nabla_{\state_{2}}\obj(\state_{1},\state_{2})),
\end{equation}
and, following standard conventions, we will treat $\gvec(\state)$ as an element of $\dspace \equiv \dual\vecspace$, the dual of the ambient space $\vecspace \equiv \vecspace_{1}\times\vecspace_{2}$, assumed to be endowed with the product norm $\norm{\state}^{2} = \norm{\state_{1}}^{2} + \norm{\state_{2}}^{2}$.

\paragraph{Variational inequalities and coherence\afterhead}

Most of the literature on \acl{SP} problems has focused on the \emph{monotone} case, \ie when $\obj$ is \emph{convex-concave}.
In such problems, it is well known that solutions of \eqref{eq:SP} can be characterized equivalently as solutions of the associated (Minty) \acl{VI}:
\begin{equation}
\label{eq:VI}
\tag{VI}
\braket{\gvec(\state)}{\state - \sol}
	\geq 0
	\quad
	\text{for all $\state\in\feas$}.
\end{equation}

Importantly, this equivalence extends well beyond the realm of monotone problems:
it trivially includes all bilinear problems ($\obj(\state_{1},\state_{2}) = \state_{1}^{\top} \mat \state_{2}$), quasi-convex-concave objectives (where Sion's minmax theorem applies), etc.
\revise{For a concrete non-monotone example, consider the problem
\begin{equation}
\label{eq:SP-example}
\min_{\state_{1}\in[-1,1]}
	\max_{\state_{2}\in[-1,1]}
		(\state_{1}^{4}\state_{2}^{2} + \state_{1}^{2} +1)
		(\state_{1}^{2}\state_{2}^{4} - \state_{2}^{2} +1).
\end{equation}
The only \acl{SP} of $\obj$ is $\sol = (0,0)$:
it is easy to check that $\sol$ is also the unique solution of the corresponding problem \eqref{eq:VI}, despite the fact that $\obj$ is not even (quasi-)monotone.}%
\footnote{To see this, simply note that $\obj(\state_{1},\state_{2})$ is \emph{multi-modal} in $\state_{2}$ for certain values of $\state_{1}$.}
This shows that the equivalence between \eqref{eq:SP} and \eqref{eq:VI} encompasses a wide range of phenomena that are innately incompatible with convexity/monotonicity, even in the lowest possible dimension;
for an in-depth discussion of the links between \eqref{eq:SP} and \eqref{eq:VI}, we refer the reader to \cite{FP03}.%


Motivated by this equivalence, we introduce below the notion of \emph{coherence:}

\begin{definition}
\label{def:coherence}
We say that \eqref{eq:SP} is \emph{coherent} if every \acl{SP} of $\obj$ is a solution of the associated \acl{VI} problem \eqref{eq:VI} and vice versa.
If \eqref{eq:VI} holds as a strict inequality whenever $\state$ is not a \acl{SP} of $\obj$, \eqref{eq:SP} will be called \emph{strictly coherent};
\revise{by contrast, if \eqref{eq:VI} holds as an equality for all $\state\in\feas$, we will say that \eqref{eq:SP} is \emph{null-coherent}.}
\end{definition}

The notion of coherence will play a central part in our considerations, so a few remarks are in order.
\revise{First, to the best of our knowledge, its first antecedent is a gradient condition examined by \cite{Bot98} in the context of nonlinear programming;
we borrow the term ``coherence'' from the more recent paper of \cite{ZMBB+17-NIPS} (who actually used the term to describe \emph{strict} coherence).
We should also note that it is possible to relax the equivalence between \eqref{eq:SP} and \eqref{eq:VI} by positing that only \emph{some} of the solutions of \eqref{eq:SP} can be harvested from \eqref{eq:VI}.}
Our analysis still goes through in this case but, to keep things simple, we do not pursue this relaxation here.

Finally, regarding the distinction between coherence and \emph{strict} coherence, we show in \cref{app:coherence} that \eqref{eq:SP} is strictly coherent when $\obj$ is strictly convex-concave.
\revise{At the other end of the spectrum, typical examples of problems that are null-coherent are bilinear objectives with an interior solution:
for instance, $\obj(\state_{1},\state_{2}) = \state_{1}\state_{2}$ with $\state_{1},\state_{2}\in[-1,1]$ has $\braket{\gvec(\state)}{\state} = \state_{1} \state_{2} - \state_{2} \state_{1} = 0$ for all $\state_{1},\state_{2}\in[-1,1]$, so it is null-coherent.
Finally, neither strict, nor null coherence imply a unique solution to \eqref{eq:SP}, a property which is particularly relevant for \acp{GAN}.}

\section{Mirror descent}
\label{sec:MD}

\paragraph{The method\afterhead}

\revise{Motivated by its prolific success in convex programming, our starting point will be the well-known \acdef{MD} method of \cite{NY83}, suitably adapted to our \acl{SP} context;}
for a survey, see \cite{Haz12} and \cite{Bub15}.

The basic idea of \acl{MD} is to generate a new state variable $\new\state$ from some starting state $\state$ by taking a ``mirror step'' along a gradient-like vector $\dstate$.
To do this, let $\hreg\from\feas\to\R$ be a continuous and $\strong$-strongly convex \acdef{DGF} on $\feas$, \ie
\begin{equation}
\label{eq:strong}
\hreg(t\state + (1-t) \alt\state)
	\leq t \hreg(\state) + (1-t) \hreg(\alt\state) - \frac{1}{2} \strong t(1-t) \norm{\alt\state - \state}^{2},
\end{equation}
for all $\state,\alt\state\in\feas$ and all $t\in[0,1]$.
In terms of smoothness (and in a slight abuse of notation), we also assume that the subdifferential of $\hreg$ admits a \emph{continuous selection}, \ie a continuous function $\nabla\hreg\from\dom\subd\hreg\to\dspace$ such that $\nabla\hreg(\state)\in\subd\hreg(\state)$ for all $\state\in\dom\subd\hreg$.%
\footnote{Recall here that the \emph{subdifferential} of $\hreg$ at $\state\in\feas$ is defined as
\(
\subd\hreg(\state)
	\equiv \setdef{\dstate\in\dspace}{\hreg(\alt\state) \geq \hreg(\state) + \braket{\dstate}{\alt\state - \state} \;\text{for all}\; \state\in\vecspace},
\)
with the standard convention $\hreg(\state) = \infty$ for all $\state\in\vecspace\setminus\feas$.}
Then, following \cite{Bre67}, $\hreg$ generates a pseudo-distance on $\feas$ via the relation
\begin{equation}
\label{eq:Bregman}
\breg(\base,\state)
	= \hreg(\base) - \hreg(\state) - \braket{\nabla\hreg(\state)}{\base - \state}
	\quad
	\text{for all $\base\in\feas$, $\state\in\dom\subd\hreg$}.
\end{equation}

This pseudo-distance is known as the \emph{Bregman divergence}.
As we show in \cref{app:Bregman}, we have $\breg(\base,\state) \geq \frac{1}{2} \strong \norm{\state - \base}^{2}$,
so the convergence of a sequence $\act_{\run}$ to some target point $\base$ can be verified by showing that $\breg(\base,\act_{\run})\to0$.
On the other hand, $\breg(\base,\state)$ typically fais to be symmetric and/or satisfy the triangle inequality, so it is not a true distance function per se.
Moreover, the level sets of $\breg(\base,\state)$ may fail to form a neighborhood basis of $\base$, so the convergence of $\act_{\run}$ to $\base$ does not necessarily imply that $\breg(\base,\act_{\run})\to0$;
we provide an example of this behavior in \cref{app:Bregman}.
For technical reasons, it will be convenient to assume that such phenomena do not occur, \ie that $\breg(\base,\act_{\run})\to0$ whenever $\act_{\run}\to\base$.
This mild regularity condition is known in the literature as ``Bregman reciprocity'' \citep{CT93,Kiw97b}, and it will be our standing assumption in what follows (note also that it holds trivially for both \cref{ex:Eucl,ex:entropic} below).

Now, as with standard Euclidean distances, the Bregman divergence generates an associated \emph{prox-mapping} defined as
\begin{equation}
\label{eq:prox}
\prox_{\state}(\dstate)
	= \argmin_{\alt\state\in\feas} \{\braket{\dstate}{\state - \alt\state} + \breg(\alt\state, \state)\}
	\quad
	\text{for all $\state\in\dom\subd\hreg$, $\dstate\in\dspace$}.
\end{equation}
In analogy with the Euclidean case (discussed below), the prox-mapping \eqref{eq:prox} produces a feasible point $\new\state = \prox_{\state}(\dstate)$ by starting from $\state\in\dom\subd\hreg$ and taking a step along a dual (gradient-like) vector $\dstate\in\dspace$.
In this way, we obtain the \acdef{MD} algorithm
\begin{equation}
\label{eq:MD}
\tag{MD}
\act_{\run+1}
	= \prox_{\act_{\run}}(-\step_{\run} \est\gvec_{\run}),
\end{equation}
where $\step_{\run}$ is a variable step-size sequence and $\hat\gvec_{\run}$ is the calculated value of the gradient vector $\gvec(\act_{\run})$ at the $\run$-th stage of the algorithm (for a pseudocode implementation, see \cref{alg:MD}).


\begin{algorithm}[tbp]
\label{alg:MD}

\small
\tt
\begin{algorithmic}[1]
\Require
	$\strong$-strongly convex regularizer $\hreg\from\feas\to\R$,
	step-size sequence $\step_{\run}>0$
\State
	choose $\act\in\dom\subd\hreg$
	\Comment{initialization}%
\For{$\run=\running$}
	\State
		oracle query at $\act$ returns $\gvec$
		\Comment{gradient feedback}%
	\State
		set $\act \leftarrow \prox_{\act}(-\step_{\run}\gvec)$
		\Comment{new state}%
\EndFor
\State
	\Return $\act$
\end{algorithmic}
\caption{\acf{MD} for \acl{SP} problems}
\end{algorithm}


For concreteness, two widely used examples of prox-mappings are as follows:
\medskip

\begin{example}[Euclidean projections]
\label{ex:Eucl}
When $\feas$ is endowed with the $L^{2}$ norm $\norm{\cdot}_{2}$, the archetypal prox-function is the (square of the) norm itself, \ie $\hreg(\state) = \frac{1}{2} \norm{\state}_{2}^{2}$.
In that case, $\breg(\base,\state) = \frac{1}{2} \norm{\state - \base}^{2}$ and the induced prox-mapping is
\begin{equation}
\label{eq:prox-Eucl}
\prox_{\state}(\dstate)
	= \Eucl(\state + \dstate),
\end{equation}
with $\Eucl(\state) = \argmin_{\alt\state\in\feas} \norm{\alt\state - \state}^{2}$ denoting the ordinary Euclidean projection onto $\feas$.
\end{example}
\medskip

\begin{example}[Entropic regularization]
\label{ex:entropic}
When $\feas$ is a $\vdim$-dimensional simplex, a widely used \ac{DGF} is the (negative) Gibbs\textendash Shannon entropy $h(\state) = \sum_{j=1}^{\vdim} \state_{j} \log \state_{j}$.
This function is $1$-strongly convex with respect to the $L^{1}$ norm \citep{SS11} and the associated pseudo-distance is the \acl{KL} divergence
\(
\dkl(\base,\state)
	= \sum_{j=1}^{\vdim} \base_{j} \log\parens{\base_{j} / \state_{j}};
\)
in turn, this yields the prox-mapping
\begin{equation}
\label{eq:prox-entropic}
\prox_{\state}(\dstate)
	= \frac{(\state_{j} \exp(\dstate_{j}))_{j=1}^{\vdim}}{\sum_{j=1}^{\vdim} \state_{j} \exp(\dstate_{j})}
	\quad
	\text{for all $\state\in\intfeas$, $\dstate\in\dspace$.}
\end{equation}
The update rule $\state \leftarrow \prox_{\state}(\dstate)$ is known in the literature as the \acdef{MW} algorithm \citep{AHK12}, and is one of the centerpieces for learning in zero-sum games \citep{FS99,MPP18,DISZ18}, adversarial bandits \citep{ACBFS95}, etc.
\end{example}
\medskip

Regarding the gradient input sequence $\est\gvec_{\run}$ of \eqref{eq:MD}, we assume that it is obtained by querying a \emph{first-order oracle}
which outputs an estimate of $\gvec(\act_{\run})$ when called at $\act_{\run}$.
This oracle could be either \emph{perfect}, returning $\est\gvec_{\run} = \gvec(\act_{\run})$ for all $\run$, or \emph{imperfect}, providing noisy gradient estimations.%
\footnote{The reason for this is that, depending on the application at hand, gradients might be difficult to compute directly \eg because they require huge amounts of data, the calculation of an unknown expectation, etc.}
By that token, we will make the following blanket assumptions for the gradient feedback sequence $\est\gvec_{\run}$:
\begin{equation}
\label{eq:oracle}
\begin{aligned}
&a)\;
	\textit{Unbiasedness:}
	&
	&\exof{\est\gvec_{\run} \given \filter_{\run}}
		= \gvec(\act_{\run}).
	\\[.5ex]
&b)\;
	\textit{Finite mean square:}
	&
	&\exof{\dnorm{\est\gvec_{\run}}^{2} \given \filter_{\run}}
		\leq \gbound^{2}
		\;\;
		\text{for some finite $\gbound\geq0$}.
		\hspace{8em}
\end{aligned}
\end{equation}

In the above, $\dnorm{\dstate} \equiv \sup\setdef{\braket{\dstate}{\state}}{\state\in\vecspace, \norm{\state} \leq 1}$ denotes the dual norm on $\dspace$
while $\filter_{\run}$ represents the history (natural filtration) of
the generating sequence $\act_{\run}$ up to stage $\run$ (inclusive).
Since $\est\gvec_{\run}$ is generated randomly from $\act_{\run}$ at stage $\run$, it is obviously not $\filter_{\run}$-measurable, \ie
\(
\est\gvec_{\run}
	= \gvec(\act_{\run}) + \noise_{\run+1},
\)
where $\noise_{\run}$ is
an adapted \acl{MDS} with $\exof{\dnorm{\noise_{\run+1}}^{2} \given \filter_{\run}} \leq \noisevar$ for some finite $\noisedev \geq 0$.
Clearly, when $\noisedev=0$,
we recover the exact gradient feedback framework $\hat\gvec_{\run} = \gvec(\act_{\run})$.

\paragraph{Convergence analysis\afterhead}

When \eqref{eq:SP} is convex-concave, it is customary to take as the output of \eqref{eq:MD} the so-called \emph{ergodic average}
\begin{equation}
\label{eq:ergodic}
\bar\act_{\run}
	= \frac{\sum_{\runalt=\start}^{\run} \step_{\runalt} \act_{\runalt}}{\sum_{\runalt=\start}^{\run} \step_{\runalt}},
\end{equation}
or some other average of the sequence $\act_{\run}$ where the objective is sampled.
The reason for this is that convexity guarantees \textendash\ via Jensen's inequality and gradient monotonicity \textendash\ that a regret-based analysis of \eqref{eq:MD} can lead to explicit rates for the convergence of $\bar\act_{\run}$ to the solution set of \eqref{eq:SP} \citep{Nem04,Nes07}.
Beyond convex-concave problems however, this is no longer the case:
averaging provides no tangible benefits in a non-monotone setting, so we need to examine the convergence properties of the generating sequence $\act_{\run}$ of \eqref{eq:MD} directly.
With all this in mind, our main result for \eqref{eq:MD} may be stated is as follows:

\revise{%
\begin{theorem}
\label{thm:MD}
Suppose that \eqref{eq:MD} is run with a gradient oracle satisfying \eqref{eq:oracle} and a variable step-size sequence $\step_{\run}$ such that $\sum_{\run=\start}^{\infty} \step_{\run} = \infty$.
Then:
\begin{enumerate}
[\itshape a\upshape)]
\item
\label{itm:strict}
If $\obj$ is strictly coherent and $\sum_{\run=\start}^{\infty} \step_{\run}^{2} < \infty$, $\act_{\run}$ converges \as to a solution of \eqref{eq:SP}.
\item
\label{itm:null}
If $\obj$ is null-coherent, the sequence $\exof{\breg(\sol,\act_{\run})}$ is non-decreasing for every solution $\sol$ of \eqref{eq:SP}.
\end{enumerate}
\end{theorem}

This result establishes an important dichotomy between strict and null coherence:
\emph{in strictly coherent problems, $\act_{\run}$ is attracted to the solution set of \eqref{eq:SP};
in null-coherent problems, $\act_{\run}$ drifts away and cycles without converging.}
In particular, this dichotomy leads to the following immediate corollaries:

\begin{corollary}
\label{cor:MD-strict}
Suppose that $\obj$ is strictly convex-concave.
Then, with assumptions as above, $\act_{\run}$ converges \as to the \textpar{necessarily unique} solution of \eqref{eq:SP}.
\end{corollary}

\begin{corollary}
\label{cor:MD-null}
Suppose that $\obj$ is bilinear
and admits an interior \acl{SP} $\sol\in\intfeas$.
If $\act_{\start}\neq\sol$ and \eqref{eq:MD} is run with exact gradient input \textpar{$\noisedev=0$}, we have $\lim_{\run\to\infty} \breg(\sol,\act_{\run}) > 0$.
\end{corollary}

Since bilinear models include all finite two-player, zero-sum games, \cref{cor:MD-null} encapsulates both the non-convergence results of \cite{DISZ18} and \cite{bailey2018multiplicative} for \acl{GD} and \ac{FTRL} respectively (for a more comprehensive formulation, see \cref{prop:finite} in \cref{app:MD}).%
} 
This failure of \eqref{eq:MD} is due to the fact that, witout a mitigating mechanism in place, a ``blind'' first-order step could overshoot and lead to an outwards spiral, even with a \emph{vanishing} step-size.
This phenomenon becomes even more pronounced in \acp{GAN} where it can lead to mode collapse and/or cycles between different modes.
The next two sections address precisely these issues.

\section{Optimistic mirror descent}
\label{sec:OMD}

\paragraph{The method\afterhead}

In convex-concave problems, taking an average of the algorithm's generated samples as in \eqref{eq:ergodic} may resolve cycling phenomena by inducing an auxiliary sequence that gravitates towards the ``center of mass'' of the driving sequence $\act_{\run}$ (which orbits interior solutions).
However, this technique cannot be employed in non-monotone problems because Jensen's inequality does not hold there.
In view of this, we replace averaging with an optimistic ``extra-gradient'' step which uses the obtained information to ``amortize'' the next prox step (possibly outside the convex hull of generated states).
The seed of this ``extra-gradient'' idea dates back to \cite{Kor76} and \cite{Nem04}, and has since found wide applications in optimization theory and beyond \textendash\ for a survey, see \cite{Bub15} and references therein.

In a nutshell, given a state $\state$, the extra-gradient method first generates an intermediate, ``waiting'' state $\est\state = \prox_{\state}(-\step\gvec(\state))$ by taking a prox step as usual.
However, instead of continuing from $\est\state$, the method samples $\gvec(\est\state)$ and goes back to the \emph{original} state $\state$ in order to generate a new state $\new\state = \prox_{\state}(-\step\gvec(\est\state))$. 
Based on this heuristic, we obtain the \acdef{OMD} algorithm
\begin{equation}
\label{eq:OMD}
\tag{OMD}
\begin{aligned}
\act_{\run+1/2}
	&= \prox_{\act_{\run}}(-\step_{\run} \est\gvec_{\run})
	\\
\act_{\run+1}
	&= \prox_{\act_{\run}}(-\step_{\run} \est\gvec_{\run+1/2})
\end{aligned}
\end{equation}
where, in obvious notation, $\est\gvec_{\run}$ and $\est\gvec_{\run+1/2}$ represent gradient oracle queries at the incumbent and intermediate states $\act_{\run}$ and $\act_{\run+1/2}$ respectively (for a pseudocode implementation, see \cref{alg:OMD}).


\begin{algorithm}[tbp]
\caption{\acf{OMD} for \acl{SP} problems}
\label{alg:OMD}

\small
\tt
\begin{algorithmic}[1]
\Require
	$\strong$-strongly convex regularizer $\hreg\from\feas\to\R$,
	step-size sequence $\step_{\run}>0$
\State
	choose $\act\in\dom\subd\hreg$
	\Comment{initialization}%
\For{$\run=\running$}
	\State
		oracle query at $\act$ returns $\gvec$
		\Comment{gradient feedback}%
	\State
		set $\new\act \leftarrow \prox_{\act}(-\step_{\run}\gvec)$
		\Comment{waiting state}%
	\State
		oracle query at $\new\act$ returns $\new\gvec$
		\Comment{gradient feedback}%
	\State
		set $\act \leftarrow \prox_{\act}(-\step_{\run}\new\gvec)$
		\Comment{new state}%
\EndFor
\State
	\Return $\act$
\end{algorithmic}
\end{algorithm}


\paragraph{Convergence analysis\afterhead}

In his original analysis, \cite{Nem04} considered the ergodic average \eqref{eq:ergodic} of the algorithm's iterates and established an $\bigoh(1/\run)$ convergence rate in monotone problems.
However, as we explained above, even though this kind of averaging is helpful in convex-concave problems, it does not provide any tangible benefits beyond this class:
in more general problems, $\act_{\run}$ appears to be the most natural solution candidate.
Our first result below justifies this choice in the class of coherent problems:

\revise{
\begin{theorem}
\label{thm:OMD}
Suppose that \eqref{eq:SP} is coherent and $\gvec$ is $\Lip$-Lipschitz continuous.
If \eqref{eq:OMD} is run with exact gradient input \textpar{$\noisedev=0$} and $\step_{\run}$ such that
\(
0
	< \inf_{\run} \step_{\run}
	\leq \sup_{\run} \step_{\run}
	< \strong/\Lip,
\)
the sequence $\act_{\run}$ converges monotonically to a solution $\sol$ of \eqref{eq:SP}, \ie $\breg(\sol,\act_{\run})$ decreases monotonically to $0$.
\end{theorem}

\begin{corollary}
\label{cor:OMD}
Suppose that $\obj$ is bilinear.
If \eqref{eq:OMD} is run with assumptions as above, the sequence $\act_{\run}$ converges monotonically to a solution of \eqref{eq:SP}.
\end{corollary}
} 

\revise{\cref{thm:OMD} includes as a special case the analysis of \citet[Theorem 12.1.11]{FP03} for \acl{OGD} and, in turn, the corresponding asymptotic result of \cite{DISZ18} for bilinear \acl{SP} problems.
As in the case of \cite{DISZ18}, \cref{thm:OMD} shows that optimism (\ie the extra-gradient add-on) plays a crucial role in stabilizing \eqref{eq:MD}:%
} 
not only does \eqref{eq:OMD} converge in problems where \eqref{eq:MD} provably fails (\eg in zero-sum finite games), but this convergence is, in fact, monotonic.
In other words, at each iteration, \eqref{eq:OMD} comes closer to a solution of \eqref{eq:SP}, whereas \eqref{eq:MD} may spiral outwards, towards higher and higher values of the Bregman divergence, ultimately converging to a limit cycle.
This phenomenon can be seen very clearly in \cref{fig:portraits}, and also in the detailed analysis we provide in \cref{app:MD}.

Of course, except for very special cases, the monotonic convergence of $\act_{\run}$ cannot hold when the gradient input to \eqref{eq:OMD} is imperfect:
a single ``bad'' sample of $\est\gvec_{\run}$ would suffice to throw $\act_{\run}$ off-track.
In this case, we have:

\begin{theorem}
\label{thm:OMD-strict}
Suppose that \eqref{eq:SP} is strictly coherent and \eqref{eq:OMD} is run with a gradient oracle satisfying \eqref{eq:oracle} and a variable step-size sequence $\step_{\run}$ such that $\sum_{\run=\start}^{\infty} \step_{\run} = \infty$ and $\sum_{\run=\start}^{\infty} \step_{\run}^{2} < \infty$.
Then, with probability $1$, $\act_{\run}$ converges to a solution of \eqref{eq:SP}.
\end{theorem}

It is worth noting here that the step-size policy in \cref{thm:OMD-strict} is different than that of \cref{thm:OMD}.
This is due to
\begin{inparaenum}
[\itshape a\upshape)]
\item
the lack of randomness (which obviates the summability requirement $\sum_{\run=\start}^{\infty} \step_{\run}^{2} < \infty$ in \cref{thm:OMD});
and
\item
the lack of Lipschitz continuity assumption (which, in the case of \cref{thm:OMD} guarantees monotonic decrease at each step, provided the step-size is not too big).
\end{inparaenum}
Importantly, the maximum allowable step-size is also controlled by the strong convexity modulus of $\hreg$, suggesting that the choice of \acl{DGF} can be fine-tuned further to allow for more aggressive step-size policies \textendash\ a key benefit of \acl{MD} methods.

\section{Experimental results}
\label{sec:experiments}

\paragraph{\aclp{GMM}\afterhead}


\begin{figure}[tbp]
\centering
\begin{subfigure}{\textwidth}
\centering
\includegraphics[width=.75\textwidth]{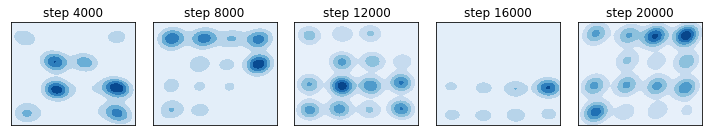}
\\
\includegraphics[width=.75\linewidth]{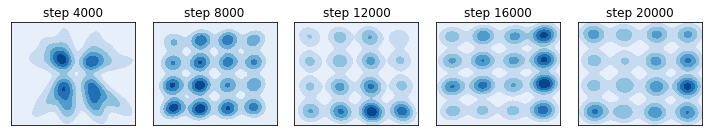}
\caption{Vanilla versus optimistic RMS (top and bottom respectively; $\step = 3\times10^{-4}$ in both cases).}
\label{fig:GMM-RMS}
\vspace{1ex}
\end{subfigure}
\begin{subfigure}{\textwidth}
\centering
\includegraphics[width=.75\textwidth]{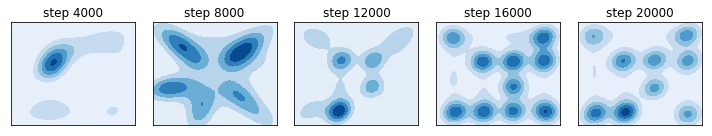}
\\
\includegraphics[width=.75\linewidth]{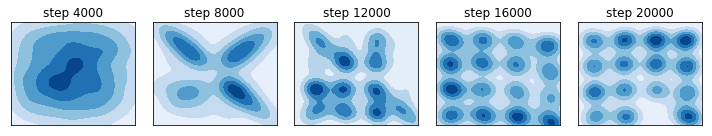}
\caption{Vanilla versus optimistic Adam (top and bottom respectively; $\step = 4\times10^{-5}$ in both cases).}
\label{fig:test-trajectories}
\vspace{-1ex}
\label{fig:GMM-Adam}
\end{subfigure}
\caption{Different algorithmic benchmarks (RMSprop and Adam):
adding an extra-gradient step allows the training method to accurately learn the target data distribution and eliminates cycling and oscillatory instabilities.
}
\label{fig:GMM}
\end{figure}


\revise{
For the experimental validation of our theoretical results, we began by evaluating the extra-gradient add-on in a highly multi-modal mixture of $16$ Gaussians arranged in a $4\times4$ grid as in \cite{unrolled2017}.
The generator and discriminator have $6$ fully connected layers with $384$ neurons and Relu activations (plus an additional layer for data space projection),
and the generator generates $2$-dimensional vectors.
The output after \{4000,  8000, 12000, 16000, 20000\} iterations is shown in \cref{fig:GMM}.
The networks were trained with RMSprop \citep{rmsprop2012} and Adam \citep{Adam2014},
and the results are compared to the corresponding extra-gradient variant (for an explicit pseudocode representation in the case of Adam, see \cite{DISZ18} and \cref{app:experiments}).
Learning rates and hyperparameters were chosen by an inspection of grid search results so as to enable a fair comparison between each method and its look-ahead version.
Overall, the different optimization strategies without look-ahead exhibit mode collapse or oscillations throughout the training period (we ran all models for at least $20000$ iterations in order to evaluate the hopping behavior of the generator).
In all cases, the extra-gradient add-on performs consistently better in learning the multi-modal distribution and greatly reduces occurrences of oscillatory behavior.
}

\paragraph{Experiments with standard datasets\afterhead}

\revise{
In our experiments with \acp{GMM}, the most promising training method was Adam with an extra-gradient step (a concrete pseudocode implementation is provided in \cref{app:experiments}).
Motivated by this, we trained a Wasserstein-\ac{GAN} on the CelebA and CIFAR-10 datasets using Adam, both with and without an extra-gradient step.
The architecture employed was a standard DCGAN;
hyperparameters and network architecture details may be found in \cref{app:experiments}.
Subsequently, to quantify the gains of the extra-gradient step, we employed the widely used inception score and Fréchet distance metrics, for which we report the results in \cref{fig:scores}.
Under both metrics, the extra-gradient add-on provides consistently higher scores after an initial warm-up period (and is considerably more stable).
For visualization purposes, we also present in \cref{fig:OptAdam} an ensemble of samples generated at the end of the training period.
Overall, the generated samples provide accurate feature representation and low distortion (especially in CelebA).
}


\begin{figure}[tbp]
\centering
\includegraphics[height=.3\textwidth]{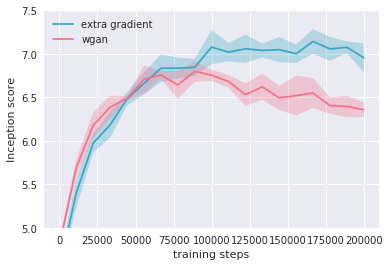}
\qquad
\includegraphics[height=.3\textwidth]{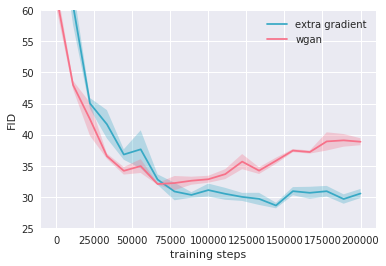}%
\vspace{-2ex}
\caption{%
Left: Inception score (left) and Fréchet distance (right) on CIFAR-10 when training with Adam (with and without an extra-gradient step).
Results are averaged over $8$ sample runs with different random seeds.
}%
\label{fig:scores}
\end{figure}



\begin{figure}[tbp]
\centering
\includegraphics[height=.3\textwidth]{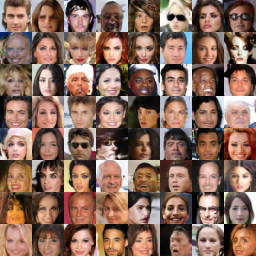}
\qquad
\includegraphics[height=.3\textwidth]{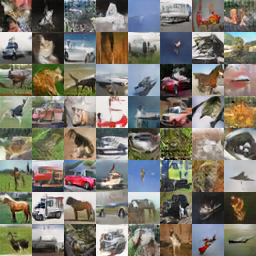}%
\vspace{-1ex}
\caption{%
Samples generated by Adam with an extra-gradient step on CelebA (left) and CIFAR-10 (right).
}%
\label{fig:OptAdam}
\vspace{-3ex}
\end{figure}


\section{Conclusions}
\label{sec:conclusion}

\revise{
Our results suggest that the implementation of an optimistic, extra-gradient step is a flexible add-on that can be easily attached to a wide variety of \ac{GAN} training methods (RMSProp, Adam, SGA, etc.), and provides noticeable gains in performance and stability.
From a theoretical standpoint, the dichotomy between strict and null coherence provides a justification of why this is so:
optimism eliminates cycles and, in so doing, stabilizes the method.
We find this property particularly appealing because it paves the way to a local analysis with provable convergence guarantees in multi-modal settings;
we intend to examine this question in future work.
}

\appendix

\section{Coherent \acl{SP} problems}
\label{app:coherence}
We begin our discussion with some basic results on coherence:

\begin{proposition}
\label{prop:convex}
If $\obj$ is convex-concave, \eqref{eq:SP} is coherent.
In addition, if $\obj$ is strictly convex-concave, \eqref{eq:SP} is strictly coherent.
\end{proposition}

\begin{proof}
Let $\sol$ be a solution point of \eqref{eq:SP}.
Since $\obj$ is convex-concave, first-order optimality gives
\begin{subequations}
\label{eq:Nash-var}
\begin{flalign}
\braket{\gvec_{1}(\sol_{1},\sol_{2})}{\state_{1} - \sol_{1}}
	&= \braket{\nabla_{\state_{1}} \obj(\sol_{1},\sol_{2})}{\state_{1} - \sol_{1}}
	\geq 0,
\intertext{and}
\braket{\gvec_{2}(\sol_{1},\sol_{2})}{\state_{2} - \sol_{2}}
	&= \braket{-\nabla_{\state_{2}} \obj(\sol_{1},\sol_{2})}{\state_{2} - \sol_{2}}
	\geq 0.
\end{flalign}
\end{subequations}
Combining the two, we readily obtain the (Stampacchia) variational inequality
\begin{equation}
\label{eq:SVI}
\braket{\gvec(\sol)}{\state - \sol}
	\geq 0
	\quad
	\text{for all $\state\in\feas$}.
\end{equation}
In addition to the above, the fact that $\obj$ is convex-concave also implies that $\gvec(\state)$ is \emph{monotone} in the sense that
\begin{equation}
\label{eq:monotone}
\braket{\gvec(\alt\state) - \gvec(\state)}{\alt\state - \state}
	\geq 0
\end{equation}
for all $\state,\alt\state\in\feas$ \cite{BC17}.
Thus, setting $\alt\state\leftarrow\sol$ in \eqref{eq:monotone} and invoking \eqref{eq:SVI}, we get
\begin{equation}
\braket{\gvec(\state)}{\state - \sol}
	\geq \braket{\gvec(\sol)}{\state -\sol}
	\geq 0,
\end{equation}
\ie \eqref{eq:VI} is satisfied.

To establish the converse implication, focus for concreteness on the minimizer, and note that \eqref{eq:VI} implies that
\begin{equation}
\label{eq:VI-1}
\braket{\gvec_{1}(\state)}{\state_{1} - \sol_{1}}
	\geq 0
	\quad
	\text{for all $\state_{1}\in\feas_{1}$}.
\end{equation}
Now, if we fix some $\state_{1}\in\feas_{1}$ and consider the function $\phi(t) = \obj(\sol_{1} + t(\state_{1} - \sol_{1}),\sol_{2})$, the inequality \eqref{eq:VI-1} yields
\begin{flalign}
\phi'(t)
	&= \braket{\gvec(\sol_{1} + t(\state_{1} - \sol_{1}),\sol_{2})}{\state_{1} - \sol_{1}}
	\notag\\
	&= \frac{1}{t} \braket{\gvec(\sol_{1} + t(\state_{1} - \sol_{1}),\sol_{2})}{\sol_{1} + t(\state_{1} - \sol_{1} - \sol_{1}}
	\geq 0,
\end{flalign}
for all $t\in[0,1]$.
This implies that $\phi$ is nondecreasing, so $\obj(\state_{1},\sol_{2}) = \phi(1) \geq \phi(0) = \obj(\sol_{1},\sol_{2})$.
The maximizing component follows similarly, showing that $\sol$ is a solution of \eqref{eq:SP} and, in turn, establishing that \eqref{eq:SP} is coherent.

For the strict part of the claim, the same line of reasoning shows that if $\braket{\gvec(\state)}{\state - \sol} = 0$ for some $\state$ that is not a \acl{SP} of $\obj$, the function $\phi(t)$ defined above must be constant on $[0,1]$, indicating in turn that $\obj$ cannot be strictly convex-concave, a contradiction.
\end{proof}

We proceed to show that the solution set of a coherent \acl{SP} problem is closed (we will need this regularity result in the convergence analysis of \cref{app:MD}):

\begin{lemma}
\label{lem:closed}
Let $\sol[\feas]$ denote the solution set of \eqref{eq:SP}.
If \eqref{eq:SP} is coherent, $\sols$ is closed.
\end{lemma}

\begin{proof}
Let $\sol_{\run}$, $\run=\running$, be a sequence of solutions of \eqref{eq:SP} converging to some limit point $\sol\in\feas$.
To show that $\sols$ is closed, it suffices to show that $\sol\in\feas$.

Indeed, given that \eqref{eq:SP} is coherent, every solution thereof satisfies \eqref{eq:VI}, so we have $\braket{\gvec(\state)}{\state - \sol_{\run}} \geq 0$ for all $\state\in\feas$.
With $\sol_{\run}\to\sol$ as $\run\to\infty$, it follows that
\begin{equation}
\braket{\gvec(\state)}{\state-\sol}
	= \lim_{\run\to\infty} \braket{\gvec(\state)}{\state - \sol_{\run}}
	\geq 0
	\quad
	\text{for all $\state\in\feas$},
\end{equation}
\ie $\sol$ satisfies \eqref{eq:VI}.
By coherence, this implies that $\sol$ is a solution of \eqref{eq:SP}, as claimed.
\end{proof}

\section{Properties of the Bregman divergence}
\label{app:Bregman}

In this appendix, we provide some auxiliary results and estimates that are used throughout the convergence analysis of \cref{app:MD}.
Some of the results we present here (or close variants thereof) are not new \citep[see \eg][]{NJLS09,JNT11}.
However, the hypotheses used to obtain them vary wildly in the literature, so we provide all the necessary details for completeness.

To begin, recall that the Bregman divergence associated to a $\strong$-strongly convex \acl{DGF} $\hreg\from\feas\to\R$ is defined as
\begin{equation}
\label{eq:app-Bregman}
\breg(\base,\state)
	= \hreg(\base) - \hreg(\state) - \braket{\nabla\hreg(\state)}{\base - \state}
\end{equation}
with $\nabla\hreg(\state)$ denoting a continuous selection of $\subd\hreg(\state)$.
The induced prox-mapping is then given by
\begin{flalign}
\label{eq:app-prox}
\prox_{\state}(\dstate)
	&= \argmin_{\alt\state\in\feas} \{ \braket{\dstate}{\state - \alt\state} + \breg(\alt\state,\state) \}
	\notag\\
	&= \argmax_{\alt\state\in\feas} \{ \braket{\dstate + \nabla\hreg(\state)}{\alt\state} - \hreg(\alt\state) \}
\end{flalign}
and is defined for all $\state\in\dom\subd\hreg$, $\dstate\in\dspace$ (recall here that $\dspace \equiv \dual\vecspace$ denotes the dual of the ambient vector space $\vecspace$).
In what follows, we will also make frequent use of the convex conjugate $\hreg^{\ast}\from\dspace\to\R$ of $\hreg$, defined as
\begin{equation}
\label{eq:app-conj}
\hreg^{\ast}(\dstate)
	= \max_{\state\in\feas} \{ \braket{\dstate}{\state} - \hreg(\state) \}.
\end{equation}
By standard results in convex analysis \cite[Chap.~26]{Roc70}, $\hreg^{\ast}$ is differentiable on $\dspace$ and its gradient satisfies the identity
\begin{equation}
\label{eq:app-dconj}
\nabla\hreg^{\ast}(\dstate)
	= \argmax_{\state\in\feas} \{ \braket{\dstate}{\state} - \hreg(\state) \}.
\end{equation}
For notational convenience, we will also write
\begin{equation}
\label{eq:app-mirror}
\mirror(\dstate)
	= \nabla\hreg^{\ast}(\dstate)
\end{equation}
and we will refer to $\mirror\from\dspace\to\feas$ as the \emph{mirror map} generated by $\hreg$.
All these notions are related as follows:
\begin{lemma}
\label{lem:mirror}
Let $\hreg$ be a \acl{DGF} on $\feas$.
Then, for all $\state\in\dom\subd\hreg$, $\dstate\in\dspace$, we have:
\begin{subequations}
\label{eq:links}
\begin{alignat}{4}
\label{eq:links-mirror}
&a)&
	&\;\;
	\state = \mirror(\dstate)
	&\;\iff\;
	&\dstate \in \subd\hreg(\state).
	\\
\label{eq:links-prox}
&b)&
	&\;\;
	\new\state = \prox_{\state}(\dstate)
	&\;\iff\;
	&\nabla\hreg(\state) + \dstate \in \subd\hreg(\new\state)
	&\;\iff\;
	&\new\state = \mirror(\nabla\hreg(\state) + \dstate).
	\hspace{4em}
\end{alignat}
\end{subequations}
Finally, if $\state = \mirror(\dstate)$ and $\base\in\feas$, we have
\begin{equation}
\label{eq:selection}
\braket{\nabla\hreg(\state)}{\state - \base}
	\leq \braket{\dstate}{\state - \base}.
\end{equation}
\end{lemma}

\begin{remark*}
By \eqref{eq:links-prox}, we have $\subd\hreg(\new\state) \neq \varnothing$, \ie $\new\state \in \dom\subd\hreg$.
As a result, the update rule $\state \leftarrow \prox_{\state}(\dstate)$ is \emph{well-posed}, \ie it can be iterated in perpetuity.
\end{remark*}

\begin{proof}[Proof of \cref{lem:mirror}]
For \eqref{eq:links-mirror}, note that $\state$ solves \eqref{eq:app-conj} if and only if $\dstate - \subd\hreg(\state) \ni 0$, \ie if and only if $\dstate\in\subd\hreg(\state)$.
Similarly, comparing \eqref{eq:app-prox} with \eqref{eq:app-conj}, it follows that $\new\state$ solves \eqref{eq:app-prox} if and only if $\nabla\hreg(\state) + \dstate \in \subd\hreg(\new\state)$, \ie if and only if $\new\state = \mirror(\nabla\hreg(\state) + \dstate)$.

For \eqref{eq:selection}, by a simple continuity argument, it suffices to show that the inequality holds for interior $\base\in\intfeas$.
To establish this, let
\begin{equation}
\phi(t)
	= \hreg(\state + t(\base-\state))
	- \bracks{\hreg(\state) +  \braket{\dstate}{\state + t(\base-\state)}}.
\end{equation}
Since $\hreg$ is strongly convex and $\dstate\in\subd\hreg(\state)$ by \eqref{eq:links-mirror}, it follows that $\phi(t)\geq0$ with equality if and only if $t=0$.
Since $\psi(t) = \braket{\nabla\hreg(\state + t(\base-\state)) - \dstate}{\base - \state}$ is a continuous selection of subgradients of $\phi$ and both $\phi$ and $\psi$ are continuous on $[0,1]$, it follows that $\phi$ is continuously differentiable with $\phi' = \psi$ on $[0,1]$.
Hence, with $\phi$ convex and $\phi(t) \geq 0 = \phi(0)$ for all $t\in[0,1]$, we conclude that $\phi'(0) = \braket{\nabla\hreg(\state) - \dstate}{\base - \state} \geq 0$, which proves our assertion.
\end{proof}

We continue with some basic bounds on the Bregman divergence before and after a prox step.
The basic ingredient for these bounds is a generalization of the (Euclidean) law of cosines which is known in the literature as the ``three-point identity'' \citep{CT93}:

\begin{lemma}
\label{lem:3points}
Let $\hreg$ be a \acl{DGF} on $\feas$.
Then, for all $\base\in\feas$ and all $\state,\alt\state\in\dom\subd\hreg$, we have
\begin{equation}
\label{eq:3points}
\breg(\base,\alt\state)
	= \breg(\base,\state)
	+ \breg(\state,\alt\state)
	+ \braket{\nabla\hreg(\alt\state) - \nabla\hreg(\state)}{\state - \base}.
\end{equation}
\end{lemma}

\begin{proof}
By definition, we have:
\begin{equation}
\begin{aligned}
\breg(\base,\alt\state)
	&= \hreg(\base) - \hreg(\alt\state) - \braket{\nabla\hreg(\alt\state)}{\base - \alt\state}
	\\
\breg(\base,\state)\hphantom{'}
	&= \hreg(\base) - \hreg(\state) - \braket{\nabla\hreg(\state)}{\base - \state}
	\\
\breg(\state,\alt\state)
	&= \hreg(\state) - \hreg(\alt\state) - \braket{\nabla\hreg(\alt\state)}{\state - \alt\state}.
\end{aligned}
\end{equation}
Our claim then follows by adding the last two lines and subtracting the first.
\end{proof}

With this identity at hand, we have the following series of upper and lower bounds:

\begin{proposition}
\label{prop:Bregman}
Let $\hreg$ be a $\strong$-strongly convex \acl{DGF} on $\feas$,
fix some $\base\in\feas$,
and let $\new\state = \prox_{\state}(\dstate)$ for $\state\in\dom\subd\hreg$, $\dstate\in\dspace$.
We then have:
\begin{subequations}
\begin{flalign}
\label{eq:Bregman-lower}
\breg(\base,\state)\hphantom{^{+}}
	&\geq \frac{\strong}{2} \norm{\state - \base}^{2}.
	\\
\label{eq:Bregman-old2new}
\breg(\base,\new\state)
	&\leq \breg(\base,\state)
	- \breg(\new\state,\state)
	+ \braket{\dstate}{\new\state - \base}
	\\
\label{eq:Bregman-old2new-alt}
	&\leq \breg(\base,\state)
	+ \braket{\dstate}{\state - \base}
	+ \frac{1}{2\strong} \dnorm{\dstate}^{2}
\end{flalign}
\end{subequations}
\end{proposition}

\begin{proof}[Proof of \eqref{eq:Bregman-lower}]
By the strong convexity of $\hreg$, we get
\begin{equation}
\hreg(\base)
	\geq \hreg(\state)
	+ \braket{\nabla\hreg(\state)}{\base - \state}
	+ \frac{\strong}{2} \norm{\base - \state}^{2}
\end{equation}
so \eqref{eq:Bregman-lower} follows by gathering all terms involving $\hreg$ and recalling the definition of $\breg(\base,\state)$.
\end{proof}

\begin{proof}[Proof of \labelcref{eq:Bregman-old2new,eq:Bregman-old2new-alt}]
By the three-point identity \eqref{eq:3points}, we readily obtain
\begin{equation}
\breg(\base,\state)
	= \breg(\base,\new\state)
	+ \breg(\new\state,\state)
	+ \braket{\nabla\hreg(\state) - \nabla\hreg(\new\state)}{\new\state - \base}.
\end{equation}
In turn, this gives
\begin{flalign}
\label{eq:upper-new2old}
\breg(\base,\new\state)
	&= \breg(\base,\state)
	- \breg(\new\state,\state)
	+ \braket{\nabla\hreg(\new\state) - \nabla\hreg(\state)}{\new\state - \base}
	\notag\\
	&\leq \breg(\base,\state)
	- \breg(\new\state,\state)
	+ \braket{\dstate}{\new\state - \base},
\end{flalign}
where, in the last step, we used \eqref{eq:selection} and the fact that $\new\state = \prox_{\state}(\dstate)$, so $\nabla\hreg(\state) + \dstate \in \subd\hreg(\new\state)$.
The above is just \eqref{eq:Bregman-old2new}, so the first part of our proof is complete.

For \eqref{eq:Bregman-old2new-alt}, the bound \eqref{eq:upper-new2old} gives
\begin{flalign}
\breg(\base,\new\state)
	&\leq \breg(\base,\state)
	+ \braket{\dstate}{\state - \base}
	+ \braket{\dstate}{\new\state - \state}
	- \breg(\new\state,\state).
\end{flalign}
Therefore, by Young's inequality \citep{Roc70}, we get
\begin{equation}
\braket{\dstate}{\new\state - \state}
	\leq \frac{\strong}{2} \norm{\new\state - \state}^{2}
	+ \frac{1}{2\strong} \dnorm{\dstate}^{2},
\end{equation}
and hence
\begin{flalign}
\breg(\base,\new\state)
	&\leq \breg(\base,\state)
	+ \braket{\dstate}{\state - \base}
	+ \frac{1}{2\strong} \dnorm{\dstate}^{2}
	+ \frac{\strong}{2} \norm{\new\state - \state}^{2}
	- \breg(\new\state,\state)
	\notag\\
	&\leq \breg(\base,\state)
	+ \braket{\dstate}{\state - \base}
	+ \frac{1}{2\strong} \dnorm{\dstate}^{2},
\end{flalign}
with the last step following from \cref{lem:mirror} applied to $\state$ in place of $\base$.
\end{proof}

The first part of \cref{prop:Bregman} shows that $\act_{\run}$ converges to $\base$ if $\breg(\base,\act_{\run})\to0$.
However, as we mentioned in the main body of the paper, the converse may fail:
in particular, we could have $\liminf_{\run\to\infty}\breg(\base,\act_{\run}) > 0$ even if $\act_{\run}\to\base$.
To see this, let $\feas$ be the $L^{2}$ ball of $\R^{\vdim}$ and take $\hreg(\state) = -\sqrt{1 - \norm{\state}_{2}^{2}}$.
Then, a straightforward calculation gives
\begin{equation}
\breg(\base,\state)
	= \frac{1 - \braket{\base}{\state}}{\sqrt{1-\norm{\state}_{2}^{2}}}
\end{equation}
whenever $\norm{\base}_{2}=1$.
The corresponding level sets $L_{c}(\base) = \setdef{\state\in\R^{\vdim}}{\breg(\base,\state) = c}$ of $\breg(\base,\cdot)$ are given by the equation
\begin{equation}
1 - \braket{\base}{\state}
	= c \sqrt{1 - \norm{\state}_{2}^{2}},
\end{equation}
which admits $\base$ as a solution for all $c\geq0$ (so $\base$ belongs to the closure of $L_{c}(\base)$ even though $\breg(\base,\base) = 0$ by definition).
As a result, under this \acl{DGF}, it is possible to have $\act_{\run}\to\base$ even when $\liminf_{\run\to\infty} \breg(\base,\act_{\run}) > 0$ (simply take a sequence $\act_{\run}$ that converges to $\base$ while remaining on the same level set of $\breg$).
As we discussed in the main body of the paper, such pathologies are discarded by the Bregman reciprocity condition
\begin{equation}
\label{eq:reciprocity}
\breg(\base,\act_{\run})
	\to 0
	\quad
	\text{whenever}
	\quad
\act_{\run}
	\to \base.
\end{equation}
This condition comes into play at the very last part of the proofs of \cref{thm:MD,thm:OMD};
other than that, we will not need it in the rest of our analysis.

%

Finally, for the analysis of the \ac{OMD} algorithm, we will need to relate prox steps taken along different directions:

\begin{proposition}
\label{prop:extra}
Let $\hreg$ be a $\strong$-strongly convex \acl{DGF} on $\feas$
and
fix some $\base\in\feas$, $\state\in\dom\subd\hreg$.
Then:
\begin{enumerate}
[\hspace{2em}\itshape a\upshape)]
\item
For all $\dstate_{1},\dstate_{2}\in\dspace$, we have:
\begin{equation}
\label{eq:prox-Lipschitz}
\norm{\prox_{\state}(\dstate_{2}) - \prox_{\state}(\dstate_{1})}
	\leq \frac{1}{\strong}\dnorm{\dstate_{2} - \dstate_{1}},
\end{equation}
\ie $\prox_{\state}$ is $(1/\strong)$-Lipschitz.

\item
In addition, letting $\new\state_{1} = \prox_{\state}(\dstate_{1})$ and $\new\state_{2} = \prox_{\state}(\dstate_{2})$, we have:
\begin{subequations}
\begin{flalign}
\label{eq:2points-prebound}
\breg(\base,\new\state_{2})
	&\leq \breg(\base,\state)
	+ \braket{\dstate_{2}}{\new\state_{1} - \base}
	+ \bracks{ \braket{\dstate_{2}}{\new\state_{2} - \new\state_{1}} - \breg(\new\state_{2},\state)}
	\\
\label{eq:2points-bound}
	&\leq \breg(\base,\state)
	+ \braket{\dstate_{2}}{\new\state_{1} - \base}
	+ \frac{1}{2\strong} \dnorm{\dstate_{2} - \dstate_{1}}^{2}
	- \frac{\strong}{2} \norm{\new\state_{1} - \state}^{2}.
\end{flalign}
\end{subequations}
\end{enumerate}
\end{proposition}

\begin{proof}
We begin with the proof of the Lipschitz property of $\prox_{\state}$.
Indeed, for all $\base\in\feas$, \eqref{eq:selection} gives
\begin{subequations}
\begin{flalign}
\label{eq:selection1}
\braket{\nabla\hreg(\new\state_{1}) - \nabla\hreg(\state) - \dstate_{1}}{\new\state_{1} - \base}
	&\leq 0,
\intertext{and}
\label{eq:selection2}
\braket{\nabla\hreg(\new\state_{2}) - \nabla\hreg(\state) - \dstate_{2}}{\new\state_{2} - \base}
	&\leq 0.
\end{flalign}
\end{subequations}
Therefore, setting $\base \leftarrow \new\state_{2}$ in \eqref{eq:selection1}, $\base \leftarrow \new\state_{1}$ in \eqref{eq:selection2} and rearranging, we obtain
\begin{equation}
\label{eq:dh-upper}
\braket{\nabla\hreg(\new\state_{2}) - \nabla\hreg(\new\state_{1})}{\new\state_{2} - \new\state_{1}}
	\leq \braket{\dstate_{2} - \dstate_{1}}{\new\state_{2} - \new\state_{1}}.
\end{equation}
By the  strong convexity of $\hreg$, we also have
\begin{equation}
\label{eq:dh-lower}
\strong \norm{\new\state_{2} - \new\state_{1}}^{2}
	\leq \braket{\nabla\hreg(\new\state_{2}) - \nabla\hreg(\new\state_{1})}{\new\state_{2} - \new\state_{1}}.
\end{equation}
Hence, combining \labelcref{eq:dh-lower,eq:dh-upper}, we get
\begin{equation}
\strong \norm{\new\state_{2} - \new\state_{1}}^{2}
	\leq \braket{\dstate_{2} - \dstate_{1}}{\new\state_{2} - \new\state_{1}}
	\leq \dnorm{\dstate_{2} - \dstate_{1}} \norm{\new\state_{2} - \new\state_{1}},
\end{equation}
and our assertion follows.

For the second part of our claim, the bound \eqref{eq:Bregman-old2new} of \cref{prop:Bregman} applied to $\new\state_{2} = \prox_{\state}(\dstate_{2})$ readily gives
\begin{flalign}
\label{eq:2points-temp}
\breg(\base,\new\state_{2})
	&\leq \breg(\base,\state)
	- \breg(\new\state_{2},\state)
	+ \braket{\dstate_{2}}{\new\state_{2} - \base}
	\notag\\
	&= \breg(\base,\state)
	+ \braket{\dstate_{2}}{\new\state_{1} - \base}
	+ \bracks{ \braket{\dstate_{2}}{\new\state_{2} - \new\state_{1}} - \breg(\new\state_{2},\state) }
\end{flalign}
thus proving \eqref{eq:2points-prebound}.
To complete our proof, note that \eqref{eq:Bregman-old2new} with $\base\leftarrow\new\state_{2}$ gives
\begin{equation}
\breg(\new\state_{2},\new\state_{1})
	\leq \breg(\new\state_{2},\state)
	+ \braket{\dstate_{1}}{\new\state_{1} - \new\state_{2}}
	- \breg(\new\state_{1},\state),
\end{equation}
or, after rearranging,
\begin{equation}
\breg(\new\state_{2},\state)
	\geq \breg(\new\state_{2},\new\state_{1})
	+ \breg(\new\state_{1},\state)
	+ \braket{\dstate_{1}}{\new\state_{2} - \new\state_{1}}.
\end{equation}
We thus obtain
\begin{flalign}
\label{eq:2points-temp2}
\braket{\dstate_{2}}{\new\state_{2} - \new\state_{1}} - \breg(\new\state_{2},\state)
	&\leq \braket{\dstate_{2} - \dstate_{1}}{\new\state_{2} - \new\state_{1}}
	- \breg(\new\state_{2},\new\state_{1})
	- \breg(\new\state_{1},\state)
	\notag\\
	&\leq \frac{\dnorm{\dstate_{2} - \dstate_{1}}^{2}}{2\strong}
	+ \frac{\strong}{2} \norm{\new\state_{2} - \new\state_{1}}^{2}
	- \frac{\strong}{2} \norm{\new\state_{2} - \new\state_{1}}^{2}
	- \frac{\strong}{2} \norm{\new\state_{1} - \state}^{2}
	\notag\\
	&\leq \frac{1}{2\strong} \dnorm{\dstate_{2} - \dstate_{1}}^{2}
	- \frac{\strong}{2} \norm{\new\state_{1} - \state}^{2},
\end{flalign}
where we used Young's inequality and \eqref{eq:Bregman-lower} in the second inequality.
The bound \eqref{eq:2points-bound} then follows by substituting \eqref{eq:2points-temp2} in \eqref{eq:2points-temp}.
\end{proof}

\section{Convergence analysis of \acl{MD}}
\label{app:MD}

We begin by recalling the definition of the \acl{MD} algorithm.
With notation as in the previous section, the algorithm is defined via the recursive scheme
\begin{equation}
\tag{MD}
\act_{\run+1}
	= \prox_{\act_{\run}}(-\step_{\run} \est\gvec_{\run}),
\end{equation}
where $\step_{\run}$ is a variable step-size sequence and $\hat\gvec_{\run}$ is the calculated value of the gradient vector $\gvec(\act_{\run})$ at the $\run$-th stage of the algorithm.
As we discussed in the main body of the paper, the gradient input sequence $\est\gvec_{\run}$ of \eqref{eq:MD} is assumed to satisfy the standard oracle assumptions
\begin{equation}
\notag
\begin{aligned}
&a)\;
	\textit{Unbiasedness:}
	&
	&\exof{\est\gvec_{\run} \given \filter_{\run}}
		= \gvec(\act_{\run}).
	\\[.5ex]
&b)\;
	\textit{Finite mean square:}
	&
	&\exof{\dnorm{\est\gvec_{\run}}^{2} \given \filter_{\run}}
		\leq \gbound^{2}
		\;\;
		\text{for some finite $\gbound\geq0$}.
		\hspace{3em}
\end{aligned}
\end{equation}
where $\filter_{\run}$ represents the history (natural filtration) of the generating sequence $\act_{\run}$ up to stage $\run$ (inclusive).

With this preliminaries at hand, our convergence proof for \eqref{eq:MD} under strict coherence will hinge on the following results:

\begin{proposition}
\label{prop:dichotomy}
Suppose that \eqref{eq:SP} is coherent and \eqref{eq:MD} is run with a gradient oracle satisfying \eqref{eq:oracle} and a variable step-size $\step_{\run}$ such that $\sum_{\run=\start}^{\infty} \step_{\run}^{2} < \infty$.
If $\sol\in\feas$ is a solution of \eqref{eq:SP}, the Bregman divergence $\breg(\sol,\act_{\run})$ converges \as to a random variable $\breg(\sol)$ with $\exof{\breg(\sol)} < \infty$.
\end{proposition}

\begin{proposition}
\label{prop:subsequence}
Suppose that \eqref{eq:SP} is strictly coherent and \eqref{eq:MD} is run with a gradient oracle satisfying \eqref{eq:oracle} and a step-size $\step_{\run}$ such that $\sum_{\run=\start}^{\infty} \step_{\run} = \infty$ and $\sum_{\run=\start}^{\infty} \step_{\run}^{2} < \infty$.
Then, with probability $1$, there exists a \textpar{possibly random} solution $\sol$ of \eqref{eq:SP} such that $\liminf_{\run\to\infty} \breg(\sol,\act_{\run}) = 0$.
\end{proposition}

\cref{prop:dichotomy} can be seen as a ``dichotomy'' result:
it shows that the Bregman divergence is an asymptotic constant of motion, so \eqref{eq:MD} either converges to a \acl{SP} $\sol$ (if $\breg(\sol) = 0$) or to some nonzero level set of the Bregman divergence (with respect to $\sol$).
In this way, \cref{prop:dichotomy} rules out more complicated chaotic or aperiodic behaviors that may arise in general \textendash\ for instance, as in the analysis of \cite{PPP17} for the long-run behavior of the \acl{MW} algorithm in two-player games.
However, unless this limit value can be somehow predicted (or estimated) in advance, this result cannot be easily applied.
This is the main role of \cref{prop:subsequence}:
it shows that \eqref{eq:MD} admits a subsequence converging to a solution of \eqref{eq:SP} so, by \eqref{eq:reciprocity}, the limit of $\breg(\sol,\act_{\run})$ must be zero.

With all this at hand, our first step is to prove \cref{prop:dichotomy}:

\begin{proof}[Proof of \cref{prop:dichotomy}]
Let $\breg_{\run} = \breg(\sol,\act_{\run})$ for some solution $\sol$ of \eqref{eq:SP}.
Then, by \cref{prop:Bregman}, we have
\begin{flalign}
\label{eq:Bregman-new-bound}
\breg_{\run+1}
	= \breg(\sol,\prox_{\act_{\run}}(-\step_{\run}\est\gvec_{\run}))
	&\leq \breg(\sol,\act_{\run})
	- \step_{\run} \braket{\est\gvec_{\run}}{\act_{\run} - \sol}
	+ \frac{\step_{\run}^{2}}{2\strong} \norm{\est\gvec_{\run}}^{2}
	\notag\\
	&= \breg_{\run}
	- \step_{\run} \braket{\gvec(\act_{\run})}{\act_{\run} - \sol}
	- \step_{\run} \braket{\noise_{\run+1}}{\act_{\run} - \sol}
	+ \frac{\step_{\run}^{2}}{2\strong} \dnorm{\est\gvec_{\run}}^{2}
	\notag\\
	&\leq \breg_{\run}
	+ \step_{\run} \snoise_{\run+1}
	+ \frac{\step_{\run}^{2}}{2\strong} \dnorm{\est\gvec_{\run}}^{2},
\end{flalign}
where, in the last line, we set $\snoise_{\run+1} = -\braket{\noise_{\run+1}}{\act_{\run} - \sol}$ and we invoked the assumption that \eqref{eq:SP} is coherent.
Thus, conditioning on $\filter_{\run}$ and taking expectations, we get
\begin{flalign}
\exof{\breg_{\run+1} \given \filter_{\run}}
	\leq \breg_{\run}
	+ \exof{\snoise_{\run+1} \given \filter_{\run}}
	+ \frac{\step_{\run}^{2}}{2\strong} \exof{\dnorm{\est\gvec_{\run}}^{2} \given \filter_{\run}}
	\leq \breg_{\run}
	+ \frac{\gbound^{2}}{2\strong} \step_{\run}^{2},
\end{flalign}
where we used the oracle assumptions \eqref{eq:oracle} and the fact that $\act_{\run}$ is $\filter_{\run}$-measurable (by definition).

Now, letting $R_{\run} = \breg_{\run} + (2\strong)^{-1} \gbound^{2} \sum_{\runalt=\run}^{\infty} \step_{\runalt}^{2}$, the estimate \eqref{eq:Bregman-new-bound} gives
\begin{equation}
\exof{R_{\run+1} \given \filter_{\run}}
	= \exof{\breg_{\run+1} \given \filter_{\run}}
	+ \frac{\gbound^{2}}{2\strong} \sum_{\runalt=\run+1}^{\infty} \step_{\runalt}^{2}
	\leq \breg_{\run}
	+ \frac{\gbound^{2}}{2\strong} \sum_{\runalt=\run}^{\infty} \step_{\runalt}^{2}
	= R_{\run},
\end{equation}
\ie $R_{\run}$ is an $\filter_{\run}$-adapted supermartingale.
Since $\sum_{\run=\start}^{\infty} \step_{\run}^{2} < \infty$, it follows that
\begin{equation}
\exof{R_{\run}}
	= \exof{\exof{R_{\run} \given \filter_{\run-1}}}
	\leq \exof{R_{\run-1}}
	\leq \dotsm
	\leq \exof{R_{\start}}
	\leq \exof{\breg_{\start}} + \frac{\gbound^{2}}{2\strong} \sum_{\run=\start}^{\infty} \step_{\run}^{2} < \infty,
\end{equation}
\ie $R_{\run}$ is uniformly bounded in $L^{1}$.
Thus, by Doob's convergence theorem for supermartingales \citep[Theorem~2.5]{HH80}, it follows that $R_{\run}$ converges \as to some finite random variable $R_{\infty}$ with $\exof{R_{\infty}} < \infty$.
In turn, by inverting the definition of $R_{\run}$, this shows that $\breg_{\run}$ converges \as to some random variable $\breg(\sol)$ with $\exof{\breg(\sol)} < \infty$, as claimed.
\end{proof}

We now turn to the proof of existence of a convergent subsequence of \eqref{eq:MD} under strict coherence (\cref{prop:subsequence}):

\begin{proof}[Proof of \cref{prop:subsequence}]
We begin with the technical observation that the solution set $\sols$ of \eqref{eq:SP} is closed \textendash\ and hence, compact (\cf \cref{lem:closed} in \cref{app:coherence}).
Clearly, if $\sols=\feas$, there is nothing to show;
hence, without loss of generality, we may assume in what follows that $\sols\neq\feas$.

Assume now ad absurdum that, with positive probability, the sequence $\act_{\run}$ generated by \eqref{eq:MD} admits no limit points in $\sols$.
Conditioning on this event, and given that $\sols$ is compact, there exists a (nonempty) compact set $\cpt\subset\feas$ such that $\cpt\cap\sols = \varnothing$ and $\act_{\run}\in\cpt$ for all sufficiently large $\run$.
Moreover, given that \eqref{eq:SP} is strictly coherent, we have $\braket{\gvec(\state)}{\state-\sol} > 0$ whenever $\state\in\cpt$ and $\sol\in\sols$.
Therefore, by the continuity of $\gvec$ and the compactness of $\sols$ and $\cpt$, there exists some $\farbound>0$ such that
\begin{equation}
\label{eq:farbound}
\braket{\gvec(\state)}{\state - \sol}	
	\geq \farbound
	\quad
	\text{for all $\state\in\cpt$, $\sol\in\feas$}.
\end{equation}

To proceed, fix some $\sol\in\sols$ and let $\breg_{\run} = \breg(\sol,\act_{\run})$.
Then, telescoping \eqref{eq:Bregman-new-bound} yields the estimate
\begin{flalign}
\label{eq:Bregman-run}
\breg_{\run+1}
	&\leq \breg_{\start}
	- \sum_{\runalt=\start}^{\run} \step_{\runalt} \braket{\gvec(\act_{\runalt})}{\act_{\runalt} - \sol}
	+ \sum_{\runalt=\start}^{\run} \step_{\runalt} \snoise_{\runalt+1}
	+ \sum_{\runalt=\start}^{\run} \frac{\step_{\runalt}^{2}}{2\strong} \dnorm{\est\gvec_{\runalt}}^{2},
\end{flalign}
where, as in the proof of \cref{prop:dichotomy}, we set $\snoise_{\run+1} = \braket{\noise_{\run+1}}{\act_{\run} - \sol}$.
Subsequently, letting $\tau_{\run} = \sum_{\runalt=\start}^{\run} \step_{\runalt}$ and using \eqref{eq:farbound}, we obtain
\begin{equation}
\breg_{\run+1}
	\leq \breg_{\start}
	- \tau_{\run} \bracks*{
		\farbound
		- \frac{\sum_{\runalt=\start}^{\run} \step_{\runalt} \snoise_{\runalt+1}}{\tau_{\run}}
		- \frac{(2\strong)^{-1} \sum_{\runalt=\start}^{\run} \step_{\runalt}^{2} \dnorm{\est\gvec_{\runalt}}^{2}}{\tau_{\run}}
	}.
\end{equation}

By the unbiasedness hypothesis of \eqref{eq:oracle} for $\noise_{\run}$, we have $\exof{\snoise_{\run+1} \given \filter_{\run}} = \braket{\exof{\noise_{\run+1}\given\filter_{\run}}}{\act_{\run} - \sol} = 0$ (recall that $\act_{\run}$ is $\filter_{\run}$-measurable by construction).
Moreover, since $\noise_{\run}$ is bounded in $L^{2}$ and $\step_{\run}$ is $\ell^{2}$ summable (by assumption), it follows that
\begin{flalign}
\sum_{\run=\start}^{\infty} \step_{\run}^{2} \exof{\snoise_{\run+1}^{2} \given \filter_{\run}}
	&\leq \sum_{\run=\start}^{\infty} \step_{\run}^{2} \norm{\act_{\run} - \sol}^{2} \exof{\dnorm{\noise_{\run+1}}^{2}\given\filter_{\run}}
	\notag\\
	&\leq \diam(\feas)^{2} \noisevar \sum_{\run=\start}^{\infty} \step_{\run}^{2}
	< \infty.
\end{flalign}
Therefore, by the law of large numbers for \aclp{MDS} \citep[Theorem~2.18]{HH80}, we conclude that $\tau_{\run}^{-1} \sum_{\runalt=\start}^{\run} \step_{\runalt} \snoise_{\runalt+1}$ converges to $0$ with probability $1$.

Finally, for the last term of \eqref{eq:Bregman-run}, let $S_{\run+1} = \sum_{\runalt=\start}^{\run} \step_{\runalt}^{2} \dnorm{\est\gvec_{\runalt}}^{2}$.
Since $\est\gvec_{\runalt}$ is $\filter_{\run}$-measurable for all $\runalt=\running,\run-1$, we have
\begin{equation}
\exof{S_{\run+1} \given \filter_{\run}}
	= \exof*{%
		\sum_{\runalt=\start}^{\run-1} \step_{\runalt}^{2} \dnorm{\est\gvec_{\runalt}}^{2}
		+ \step_{\run}^{2} \dnorm{\est\gvec_{\run}}^{2} \given \filter_{\run}
		}
	= S_{\run} + \step_{\run}^{2} \exof{\dnorm{\est\gvec_{\run}}^{2} \given \filter_{\run}}
	\geq S_{\run},
\end{equation}
\ie $S_{\run}$ is a submartingale with respect to $\filter_{\run}$.
Furthermore, by the law of total expectation, we also have
\begin{equation}
\exof{S_{\run+1}}
	= \exof{\exof{S_{\run+1} \given \filter_{\run}}}
	\leq \gbound^{2} \sum_{\runalt=\start}^{\run} \step_{\run}^{2}
	\leq \gbound^{2} \sum_{\runalt=\start}^{\infty} \step_{\run}^{2}
	< \infty,
\end{equation}
so $S_{\run}$ is bounded in $L^{1}$.
Hence, by Doob's submartingale convergence theorem \citep[Theorem~2.5]{HH80}, we conclude that $S_{\run}$ converges to some (almost surely finite) random variable $S_{\infty}$ with $\exof{S_{\infty}} < \infty$, implying in turn that $\lim_{\run\to\infty} S_{\run+1}/\tau_{\run} = 0$ \as.

Applying all of the above, the estimate \eqref{eq:Bregman-run} gives $\breg_{\run+1} \leq \breg_{\start} - \farbound \tau_{\run} / 2$ for sufficiently large $\run$, so $\breg(\sol,\act_{\run}) \to -\infty$, a contradiction.
Going back to our original assumption, this shows that, with probability $1$, at least one of the limit points of $\act_{\run}$ must lie in $\sols$, as claimed.
\end{proof}

With all this at hand, we are finally in a position to prove our main result for \eqref{eq:MD}:

\begin{proof}[Proof of \cref{thm:MD}\textpar{\ref{itm:strict}}]
\cref{prop:subsequence} shows that, with probability $1$, there exists a \textpar{possibly random} solution $\sol$ of \eqref{eq:SP} such that $\liminf_{\run\to\infty} \norm{\act_{\run} - \sol} = 0$ and, hence, $\liminf_{\run\to\infty} \breg(\sol,\act_{\run}) = 0$ (by Bregman reciprocity).
Since $\lim_{\run\to\infty} \breg(\sol,\act_{\run})$ exists with probability $1$ (by \cref{prop:dichotomy}), it follows that
\(
\lim_{\run\to\infty} \breg(\sol,\act_{\run})
	= \liminf_{\run\to\infty} \breg(\sol,\act_{\run})
	= 0,
\)
\ie $\act_{\run}$ converges to $\sol$.
\end{proof}

We proceed with the negative result hinted at in the main body of the paper, namely the failure of \eqref{eq:MD} to converge under null coherence:

\begin{proof}[Proof of \cref{thm:MD}\textpar{\ref{itm:null}}]
The evolution of the Bregman divergence under \eqref{eq:MD} satisfies the identity
\begin{flalign}
\breg(\sol,\act_{\run+1})
	&= \breg(\sol,\act_{\run})
	+ \breg(\act_{\run},\act_{\run+1})
	+ \step_{\run} \braket{\est\gvec_{\run}}{\act_{\run} - \sol}
	\notag\\
	&= \breg(\sol,\act_{\run}) + \breg(\act_{\run},\act_{\run+1})
	+ \braket{\noise_{\run+1}}{\act_{\run} - \sol}
\end{flalign}
where, in the last line, we used the null coherence assumption $\braket{\gvec(\state)}{\state-\sol} = 0$ for all $\state\in\feas$.
Since $\breg(\act_{\run},\act_{\run+1})\geq0$, taking expecations above shows that $\breg(\sol,\act_{\run})$ is nondecreasing, as claimed.
\end{proof}

With \cref{thm:MD} at hand, the proof of \cref{cor:MD-strict} is an immediate consequence of the fact that strictly convex-concave problems satisfy strict coherence (\cref{prop:convex}).
As for \cref{cor:MD-null}, we provide below a more general result for two-player, zero-sum finite games.

To state it, let $\pures_{\play} = \{1,\dotsc,\nPures_{\play}\}$, $\play=1,2$, be two finite sets of \emph{pure strategies}, and let $\feas_{\play} = \simplex(\pures_{\play})$ denote the set of \emph{mixed strategies} of player $\play$.
A \emph{finite, two-player zero-sum game} is then defined by a matrix $\mat\in\R^{\pures_{1}\times\pures_{2}}$ so that the loss of Player $1$ and the reward of Player $2$ in the mixed strategy profile $\state = (\state_{1},\state_{2})\in\feas$ are concurrently given by
\begin{equation}
\label{eq:finite}
\obj(\state_{1},\state_{2})
	= \state_{1}^{\top} \mat \state_{2}
\end{equation}
Then, writing $\fingame \equiv \fingame(\pures_{1},\pures_{2},\mat)$ for the resulting game, we have:

\begin{proposition}
\label{prop:finite}
Let $\fingame$ be a two-player zero-sum game with an interior \acl{NE} $\eq$.
If $\act_{\start} \neq \eq$ and \eqref{eq:MD} is run with exact gradient input \textpar{$\noisevar=0$},
we have $\lim_{\run\to\infty} \breg(\eq,\act_{\run})>0$.
If, in addition, $\sum_{\run=\start}^{\infty} \step_{\run}^{2} < \infty$, $\lim_{\run\to\infty} \breg(\sol,\act_{\run})$ is finite.
\end{proposition}

\begin{remark*}
Note that non-convergence does not require any summability assumptions on $\step_{\run}$.
\end{remark*}

In words, \cref{prop:finite} states that \eqref{eq:MD} does not converge in finite zero-sum games with a unique interior equilibrium and exact gradient input:
instead, $\act_{\run}$ cycles at positive Bregman distance from the game's \acl{NE}.
Heuristically, the reason for this behavior is that, for small $\step\to0$, the incremental step $\dynfield_{\step}(\state) = \prox_{\state}(-\step\gvec(\state)) - \state$ of \eqref{eq:MD} is essentially tangent to the level set of $\breg(\sol,\cdot)$ that passes through $\state$.%
\footnote{This observation was also the starting point of \cite{MPP18} who showed that \ac{FTRL} in \emph{continuous time} exhibits a similar cycling behavior in zero-sum games with an interior equilibrium.}
For finite $\step>0$, things are even worse because $\dynfield_{\step}(\state)$ points noticeably away from $\state$, 
\ie towards higher level sets of $\breg$.
As a result, the ``best-case scenario'' for \eqref{eq:MD} is to orbit $\sol$ (when $\step\to0$);
in practice, for finite $\step$, the algorithm takes small outward steps throughout its runtime, eventually converging to some limit cycle farther away from $\sol$.

\begin{figure}[tbp]
\centering
\includegraphics[height=.45\textwidth]{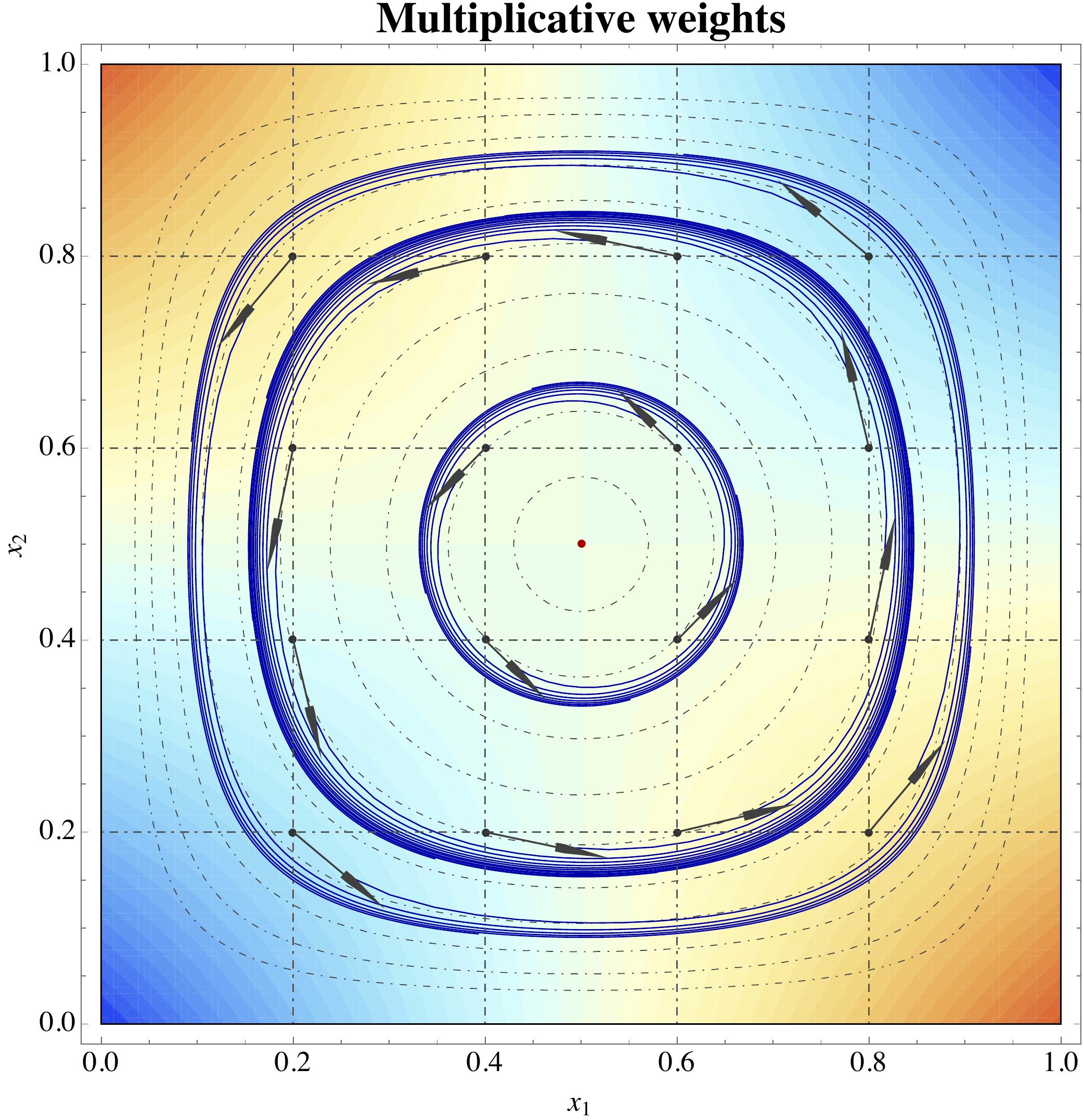}
\hfill
\includegraphics[height=.45\textwidth]{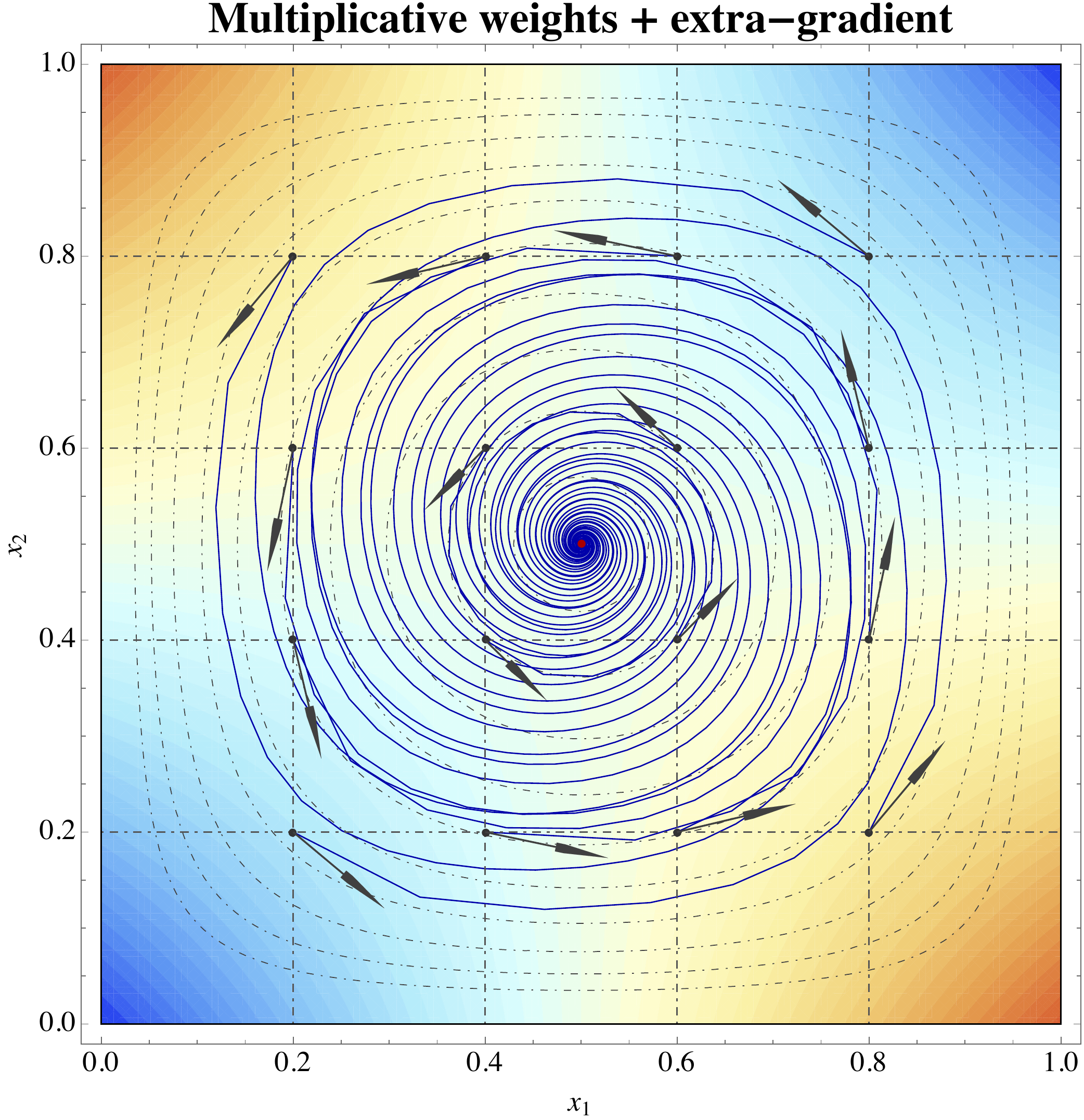}%
\caption{%
Trajectories of vanilla and optimistic \acl{MD} in a zero-sum game of Matching Pennies (left and right respectively).
Colors represent the contours of the objective, $\obj(\state_{1},\state_{2}) = (\state_{1} - 1/2) (\state_{2} - 1/2)$.
}%
\vspace{-2ex}
\label{fig:portraits-bilinear}
\end{figure}

We make this intuition precise below (for a schematic illustration, see also \cref{fig:portraits} above):

\begin{proof}[Proof of \cref{prop:finite}]
Write $\payv_{1}(\state) = -\mat \state_{2}$ and $\payv_{2}(\state) = \state_{1}^{\top}\mat$ for the players' payoff vectors under the mixed strategy profile $\state = (\state_{1},\state_{2})$.
By construction, we have $\gvec(\state) = -(\payv_{1}(\state),\payv_{2}(\state))$.
Furthermore, since $\sol$ is an interior equilibrium of $\obj$, elementary game-theoretic considerations show that $\payv_{1}(\sol)$ and $\payv_{2}(\sol)$ are both proportional to the constant vector of ones.
We thus get
\begin{flalign}
\braket{\gvec(\state)}{\state-\sol}
	&= \braket{\payv_{1}(\state)}{\state_{1}-\sol_{1}}
	+ \braket{\payv_{2}(\state)}{\state_{2} - \sol_{2}}
	\notag\\
	&= -\state_{1}^{\top} \mat \state_{2} + (\sol_{1})^{\top} \mat \state_{2}
	+ \state_{1}^{\top} \mat \state_{2} - \state_{1}^{\top}\mat \sol_{2}
	\notag\\
	&= 0,
\end{flalign}
where, in the last line, we used the fact that $\sol$ is interior.
This shows that $\obj$ satisfies null coherence, so our claim follows from \cref{thm:MD}\textpar{\ref{itm:null}}.

For our second claim, arguing as above and using \eqref{eq:Bregman-old2new-alt}, we get
\begin{flalign}
\breg(\sol,\act_{\run+1})
	&\leq \breg(\sol,\act_{\run})
	+ \step_{\run}\braket{\gvec(\act_{\run})}{\act_{\run} - \sol}
	+ \frac{\step_{\run}^{2}}{2\strong} \dnorm{\gvec(\act_{\run})}^{2}
	\notag\\
	&\leq \breg(\sol,\act_{\run}) + \frac{\step_{\run}^{2}\gbound^{2}}{2\strong}
\end{flalign}
with $\gbound = \max_{\state_{1}\in\feas_{1},\state_{2}\in\feas_{2}} \dnorm{(-\mat \state_{2},\state_{1}^{\top}\mat)}$.
Telescoping this last bound yields
\begin{equation}
\sup_{\run} \breg(\sol,\act_{\run})
	\leq \breg(\sol,\act_{\start}) + \sum_{\runalt=\start}^{\infty} \frac{\step_{\run}^{2}\gbound^{2}}{2\strong}
	< \infty,
\end{equation}
so $\breg(\sol,\act_{\run})$ is also bounded from above.
Therefore, with $\breg(\sol,\act_{\run})$ nondecreasing, bounded from above and $\breg(\sol,\act_{\start})>0$, it follows that $\lim_{\run\to\infty} \breg(\sol,\act_{\run}) > 0$, as claimed.
\end{proof}

\section{Convergence analysis of \acl{OMD}}
\label{app:OMD}

We now turn to the \acdef{OMD} algorithm, as defined by the recursion
\begin{equation}
\tag{OMD}
\begin{aligned}
\act_{\run+1/2}
	&= \prox_{\act_{\run}}(-\step_{\run} \est\gvec_{\run})
	\\
\act_{\run+1}
	&= \prox_{\act_{\run}}(-\step_{\run} \est\gvec_{\run+1/2})
\end{aligned}
\end{equation}
with
$\act_{\start}$ initialized arbitrarily in $\dom\subd\hreg$,
and
$\est\gvec_{\run}$, $\est\gvec_{\run+1/2}$ representing gradient oracle queries at the incumbent and intermediate states $\act_{\run}$ and $\act_{\run+1/2}$ respectively.

The heavy lifting for our analysis is provided by \cref{prop:extra}, which leads to the following crucial lemma:

\begin{lemma}
\label{lem:descent}
Suppose that \eqref{eq:SP} is coherent and $\gvec$ is $\Lip$-Lipschitz continuous.
With notation as above and exact gradient input \textpar{$\noisedev=0$}, we have
\begin{equation}
\label{eq:descent}
\breg(\sol,\act_{\run+1})
	\leq \breg(\sol,\act_{\run})
	- \frac{1}{2} \parens*{\strong - \frac{\step_{\run}^{2}\Lip^{2}}{\strong}} \norm{\act_{\run+1/2} - \act_{\run}}^{2},
\end{equation}
for every solution $\sol$ of \eqref{eq:SP}.
\end{lemma}

\begin{proof}
Substituting
$\state \leftarrow \act_{\run}$,
$\dstate_{1} \leftarrow -\step_{\run}\gvec(\act_{\run})$,
and
$\dstate_{2} \leftarrow -\step_{\run}\gvec(\act_{\run+1/2})$
in \cref{prop:extra}, we obtain the estimate:
\begin{flalign}
\breg(\sol,\act_{\run+1})
	&\leq \breg(\sol,\act_{\run})
	- \step_{\run} \braket{\gvec(\act_{\run+1/2})}{\act_{\run+1/2} - \sol}
	\notag\\
	&+ \frac{\step_{\run}^{2}}{2\strong} \dnorm{\gvec(\act_{\run+1/2}) - \gvec(\act_{\run})}^{2}
	- \frac{\strong}{2} \norm{\act_{\run+1/2} - \act_{\run}}^{2}
	\notag\\
	&\leq \breg(\sol,\act_{\run})
	+ \frac{\step_{\run}^{2}\Lip^{2}}{2\strong} \norm{\act_{\run+1/2} - \act_{\run}}^{2}
	- \frac{\strong}{2} \norm{\act_{\run+1/2} - \act_{\run}}^{2},
\end{flalign}
where, in the last line, we used the fact that $\sol$ is a solution of \eqref{eq:SP}/\eqref{eq:VI}, and that $\gvec$ is $\Lip$-Lipschitz.
\end{proof}

We are now finally in a position to prove \cref{thm:OMD} (reproduced below for convenience):

\begin{theorem*}
Suppose that \eqref{eq:SP} is coherent and $\gvec$ is $\Lip$-Lipschitz continuous.
If \eqref{eq:OMD} is run with exact gradient input and a step-size sequence $\step_{\run}$ such that
\begin{equation}
\label{eq:step}
\txs
0
	< \lim_{\run\to\infty}\step_{\run}
	\leq \sup_{\run} \step_{\run}
	< \strong/\Lip,
\end{equation}
the sequence $\act_{\run}$ converges monotonically to a solution $\sol$ of \eqref{eq:SP}, \ie $\breg(\sol,\act_{\run})$ is non-increasing and converges to $0$.
\end{theorem*}

\begin{proof}
Let $\sol$ be a solution of \eqref{eq:SP}.
Then, by the stated assumptions for $\step_{\run}$, \cref{lem:descent} yields
\begin{flalign}
\breg(\sol,\act_{\run+1})
	\leq \breg(\sol,\act_{\run})
	- \frac{1}{2} \strong (1 - \alpha^{2}) \norm{\act_{\run+1/2} - \act_{\run}}^{2},
\end{flalign}
where $\alpha\in(0,1)$ is such that $\step_{\run}^{2} < \alpha \strong/\Lip$ for all $\run$ (that such an $\alpha$ exists is a consequence of the assumption that $\sup_{\run} \step_{\run} < \strong/\Lip$).
This shows that $\breg(\sol,\act_{\run})$ is non-decreasing for every solution $\sol$ of \eqref{eq:SP}.

Now, telescoping \eqref{eq:descent}, we obtain
\begin{equation}
\breg(\sol,\act_{\run+1})
	\leq \breg(\sol,\act_{\start})
	- \frac{1}{2} \sum_{\runalt=\start}^{\run} \parens*{\strong - \frac{\step_{\runalt}^{2}\Lip^{2}}{\strong}} \norm{\act_{\runalt+1/2} - \act_{\runalt}}^{2},
\end{equation}
and hence:
\begin{equation}
\sum_{\runalt=\start}^{\run} \parens*{1 - \frac{\step_{\runalt}^{2} \Lip^{2}}{\strong^{2}}} \norm{\act_{\runalt+1/2} - \act_{\runalt}}^{2}
	\leq \frac{2}{\strong} \breg(\sol,\act_{\start}).
\end{equation}
With $\sup_{\run} \step_{\run} < \strong/\Lip$, the above estimate readily yields $\sum_{\run=1}^{\infty} \norm{\act_{\run+1/2} - \act_{\run}}^{2} < \infty$, which in turn implies that $\norm{\act_{\run+1/2} - \act_{\run}} \to 0$ as $\run\to\infty$.

By the compactness of $\feas$, we further infer that $\act_{\run}$ admits an accumulation point $\acc$, \ie there exists a subsequence $\run_{\runalt}$ such that $\act_{\run_{\runalt}} \to \acc$ as $\runalt\to\infty$.
Since $\norm{\act_{\run_{\runalt}+1/2} - \act_{\run_{\runalt}}} \to 0$, this also implies that $\act_{\run_{\runalt}+1/2}$ converges to $\acc$ as $\runalt\to\infty$.
Further, by passing to a subsequence if necessary, we may also assume without loss of generality that $\step_{\run_{\runalt}}$ converges to some limit value $\step>0$.
Then, by the Lipschitz continuity of the prox-mapping (\cf \cref{prop:extra}), we readily obtain
\begin{equation}
\acc
	= \lim_{\runalt\to\infty} \act_{\run_{\runalt}+1/2}
	= \lim_{\runalt\to\infty} \prox_{\act_{\run_{\runalt}}}(\act_{\run_{\runalt}} - \step_{\run_{\runalt}} \gvec(\act_{\run_{\runalt}}))
	= \prox_{\acc}(\acc - \step \gvec(\acc)),
\end{equation}
\ie $\acc$ is a solution of \eqref{eq:VI} \textendash\ and, hence, \eqref{eq:SP}.
Since $\breg(\acc,\act_{\run})$ is nonincreasing and $\liminf_{\run\to\infty} \breg(\acc,\act_{\run}) = 0$ (by the Bregman reciprocity requirement), we conclude that $\liminf_{\run\to\infty} \breg(\acc,\act_{\run}) = 0$, \ie $\act_{\run}$ converges to $\acc$.
Since $\acc$ is a solution of \eqref{eq:SP}, our proof is complete.
\end{proof}

Our last result concerns the convergence of \eqref{eq:OMD} in strictly coherent problems with a stochastic gradient oracle:

\begin{proof}[Proof of \cref{thm:OMD-strict}]
Our argument hinges on the inequality
\begin{equation}
\breg(\sol,\act_{\run+1})
	\leq \breg(\sol,\act_{\run})
	- \step_{\run} \braket{\est\gvec_{\run+1/2}}{\act_{\run+1/2} - \sol}
	+ \step_{\run}^{2} / (2\strong) \, \dnorm{\est\gvec_{\run+1/2} - \est\gvec_{\run}}^{2}
\end{equation}
which is obtained from the two-point estimate \eqref{eq:2points-bound} by substituting
$\state \leftarrow \sol$,
$\state_{1} \leftarrow \act_{\run}$,
$\dstate_{1} \leftarrow \est\gvec_{\run}$,
$\new\state_{1}\leftarrow \act_{\run+1/2} = \prox_{\act_{\run}}(-\step_{\run}\est\gvec_{\run})$,
$\dstate_{2} \leftarrow \est\gvec_{\run+1/2}$,
and
$\new\state_{2} \leftarrow \act_{\run} = \prox_{\act_{\run}}(-\step_{\run} \est\gvec_{\run+1/2})$.
Then, working as in the proof of \cref{prop:dichotomy}, we obtain the following estimate for the sequence $\breg_{\run} = \breg(\sol,\act_{\run})$:
\begin{flalign}
\label{eq:OMD-basicbound}
\breg_{\run+1}
	&\leq \breg_{\run}
	- \step_{\run} \braket{\gvec(\act_{\run+1/2})}{\act_{\run+1/2} - \sol}
	- \step_{\run} \braket{\new\noise_{\run+1}}{\act_{\run} - \sol}
	+ \frac{\step_{\run}^{2}}{2\strong} \dnorm{\est\gvec_{\run+1/2} - \est\gvec_{\run}}^{2}
	\notag\\
	&\leq \breg_{\run}
	+ \step_{\run} \new\snoise_{\run+1}
	+ \frac{\step_{\run}^{2}}{\strong} \bracks*{\dnorm{\est\gvec_{\run}}^{2} + \dnorm{\est\gvec_{\run+1/2}^{2}}},
\end{flalign}
where $\new\noise_{\run+1} = \est\gvec_{\run+1/2} - \gvec(\act_{\run+1/2})$ denotes the martingale part of $\est\gvec_{\run+1/2}$ and we have set $\new\snoise_{\run+1} = \braket{\new\noise_{\run+1}}{\act_{\run+1/2} - \sol}$.
Since $\exof{\dnorm{\gvec_{\run}}^{2} \given \act_{\run},\dotsc,\act_{\start}}$ and $\exof{\dnorm{\gvec_{\run+1/2}}^{2} \given \act_{\run+1/2},\dotsc,\act_{\start}}$ are both bounded by $\gbound^{2}$, we get the bound
\begin{equation}
\exof{\breg_{\run+1} \given \filter_{\run}}
	\leq \breg_{\run} + \frac{\gbound}{\strong} \step_{\run}^{2}.
\end{equation}
Then, following the same steps as in the proof of \cref{prop:dichotomy}, it follows that $\breg_{\run}$ converges to some limit value $\breg_{\infty}$.

To proceed, telescoping \eqref{eq:OMD-basicbound} also yields
\begin{equation}
\breg_{\run+1}
	\leq \breg_{\start}
	- \sum_{\runalt=\start}^{\run} \step_{\runalt} \braket{\gvec(\act_{\runalt+1/2})}{\act_{\runalt+1/2} - \sol}
	+ \sum_{\runalt=\start}^{\run} \step_{\runalt} \new\snoise_{\runalt+1}
	+ \sum_{\runalt=\start}^{\run} \frac{\step_{\runalt}^{2}}{2\strong} \dnorm{\est\gvec_{\runalt+1/2} - \est\gvec_{\runalt}}^{2}.
\end{equation}
Each term in the above bound can be controlled in the same way as the corresponding terms in \eqref{eq:Bregman-run}.
Thus, repeating the steps in the proof of \cref{prop:subsequence}, it follows that there exists a subsequence of $\act_{\run+1/2}$ (and hence also of $\act_{\run}$) which converges to $\sol$.

Our claim then follows by combining the two intermediate results above in the same way as in the proof of \cref{thm:MD}\textpar{\ref{itm:strict}};
to avoid needless repetition, we omit the details.
\end{proof}

\section{Experimental results}
\label{app:experiments}

\subsection{Adam with extra-gradient step}

For most of our experiments, the method that seemed to generate the best results was Adam and its optimistic version \citep{DISZ18};
for a pseudocode iplementation, see \cref{alg:extra_Adam} below.
We also noticed empirically that it was more efficient to use two different sets of moment estimates  $(m_t,v_t)$ and $(m_t',v_t')$ for the first and the second gradient steps.
We used this algorithm for our experiments with both \acp{GMM} and the CelebA/CIFAR-10 datasets.

\begin{algorithm}
\flushleft
Compute stochastic gradient: $\nabla_{\theta,t}$\;\\
Update biased estimate of 1st momentum: $m_t=\beta_1m_{t-1}+(1-\beta_1) \nabla_{\theta,t}$\;\\
Update biased estimate of 2nd momentum: $v_t=\beta_2v_{t-1}+(1-\beta_2) \nabla_{\theta,t}^2$ \; \\
Compute bias corrected 1st moment: $\hat{m}_t = \frac{m_t}{1-\beta_1^t}$\;\\
Compute bias corrected 2nd moment: $\hat{v}_t = \frac{v_t}{1-\beta_2^t}$\;\\
Perform: $\theta^{'}_t=\theta_{t-1}-\eta \frac{\hat{m}_t}{\sqrt{\hat{v}_t}+\epsilon}$\;\\
Compute stochastic gradient: $\nabla_{\theta',t}$\;\\
Update biased estimate of 1st momentum: $m_t'=\beta_1m_{t-1}'+(1-\beta_1) \nabla_{\theta',t}$\;\\
Update biased estimate of 2nd momentum: $v_t'=\beta_2v_{t-1}'+(1-\beta_1) \nabla_{\theta',t}^2$ \; \\
Compute bias corrected 1st moment: $\hat{m}_t' = \frac{m_t'}{1-\beta_1^t}$\;\\
Compute bias corrected 2nd moment: $\hat{v}_t' = \frac{v_t'}{1-\beta_1^t}$\;\\
Perform: $\theta_t=\theta_{t-1}-\eta' \frac{\hat{m}'_t}{\sqrt{\hat{v}'_t}+\epsilon}$\;\\
Return $\theta_t$
    \caption{Adam with extra-gradient add-on (optimistic Adam)}
\label{alg:extra_Adam}
\end{algorithm}

\subsection{Experiments with standards datasets}


In this section we present the results of our image experiments using \ac{OMD} training techniques.
Inception and FID scores obtained by our model during training were reported in \cref{fig:scores}:
as can be seen there, the extra-gradient add-on improves the performance of \ac{GAN} training and efficiently stabilizes the model;
without the extra-gradient step, performance tends to drop noticeably after approximately $100k$ steps.

For ease of comparison, we provide below a collection of samples generated by Adam and optimistic Adam in the CelebA and CIFAR-10 datasets.
Especially in the case of CelebA, the generated samples are consistently more representative and faithful to the target data distribution.


\begin{figure}[t]
\centering
\begin{subfigure}{\textwidth}
\centering
\includegraphics[width=.46\textwidth]{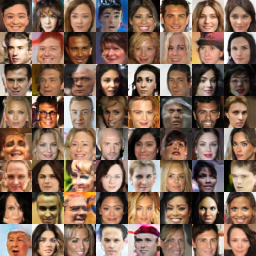}
\hfill
\includegraphics[width=.46\textwidth]{Figures/faces.png}
\caption{Vanilla versus optimistic Adam training in the CelebA dataset (left and right respectively).}
\medskip
\label{fig:CelebA}
\end{subfigure}
\begin{subfigure}{\textwidth}
\centering
\includegraphics[width=.46\textwidth]{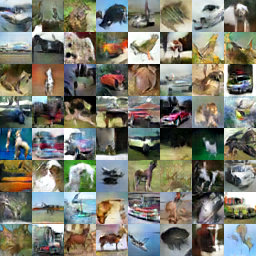}
\hfill
\includegraphics[width=.46\textwidth]{Figures/cifar.png}
\caption{Vanilla versus optimistic Adam training in the CIFAR-10 dataset (left and right respectively).}
\label{fig:test-trajectories}
\label{fig:CIFAR}
\end{subfigure}
\caption{\ac{GAN} training with and without an extra-gradient step in the CelebA and CIFAR-10 datasets.
}
\label{fig:comparison}
\end{figure}


%
%
%
%
%
%
%

\subsubsection{Network Architecture and hyperparameters}

For the reproducibility of our experiments, we provide \cref{archi2} and \cref{table:hyper} the network architectures and the hyperparameters of the \acp{GAN} that we used.
The architecture employed is a standard DCGAN architecture with a $5$-layer generator with batchnorm, and an $8$-layer discriminator.
The generated samples were 32$\times$32$\times$3 RGB images.

\begin{table}[tbp]
\centering
\caption{Generator and discriminator architectures for our images experiments}
\label{archi2}
\begin{tabular}{c}
\hline
Generator                 \\ \hline
latent space 100 (gaussian noise)                            \\
dense 4 $\times$ 4 $\times$ 512 batchnorm ReLU  \\
4$\times$4 conv.T stride=2 256 batchnorm ReLU \\
4$\times$4 conv.T stride=2 128 batchnorm ReLU \\
4$\times$4 conv.T stride=2 64 batchnorm ReLU \\
4$\times$4 conv.T stride=1 3 weightnorm tanh \\ 
\\ \hline
Discriminator                 \\ \hline
Input Image 32$\times$32$\times$3 \\ 
3$\times$3 conv. stride=1 64 lReLU \\
3$\times$3 conv. stride=2 128 lReLU\\
3$\times$3 conv. stride=1 128 lReLU\\
3$\times$3 conv. stride=2 256 lReLU\\
3$\times$3 conv. stride=1 256 lReLU\\
3$\times$3 conv. stride=2 512 lReLU\\
3$\times$3 conv. stride=1 512 lReLU\\
dense 1
\end{tabular}
\end{table}

\begin{table}[tbp]
\caption{Image experiments settings}
\label{table:hyper}
\begin{center}
\begin{tabular}{ l } 
 \toprule
    batch size = 64
    Adam learning rate = 0.0001 \\
    Adam $\beta_{1}=0.0$ \\
    Adam $\beta_{2}=0.9$ \\
    max iterations = 200000 \\
    WGAN-GP $\lambda=1.0$ \\
    WGAN-GP $n_{dis}=1$ \\
    GAN objective = 'WGAN-GP' \\
    Optimizer = 'extra-Adam' or 'Adam' \\
 \bottomrule
\end{tabular}
\end{center}
\end{table}

\bibliographystyle{ormsv080}
\bibliography{IEEEabrv,../bibtex/Bibliography-Saddle,../bibtex/refer,../bibtex/gradient,../bibtex/sample}

\end{document}